\documentclass{article} % For LaTeX2e
\usepackage[margin=1in]{geometry}
\usepackage{authblk}
\usepackage{amsmath,amsfonts,bm}

\def\eqref#1{equation~\ref{#1}}
\def\1{\bm{1}}

\DeclareMathAlphabet{\mathsfit}{\encodingdefault}{\sfdefault}{m}{sl}
\SetMathAlphabet{\mathsfit}{bold}{\encodingdefault}{\sfdefault}{bx}{n}

\usepackage{natbib}
\usepackage[utf8]{inputenc} % allow utf-8 input
\usepackage{hyperref}
\usepackage{url}
\usepackage{amsfonts}       % blackboard math symbols
\usepackage{amstext}
\usepackage{amsmath,amsthm}
\usepackage{xcolor}         % colors
\usepackage[linesnumbered,ruled]{algorithm2e}   % algorithms
\usepackage{multirow}
\usepackage{graphicx}
\usepackage{comment}
\usepackage{tablefootnote}

 \newtheorem{theorem}{\textbf{Theorem}}
 \newtheorem{corollary}{\textbf{Corollary}}
 \newtheorem{proposition}{\textbf{Proposition}} \newtheorem{lemma}{\textbf{Lemma}}
 \newtheorem{definition}{\textbf{Definition}}
  \newcommand{\dist}{\mathtt{dist}}
 \newtheorem{problem}{\textbf{Problem}}
 \newtheorem{observation}{\textbf{Observation}}

\newtheorem{remark}{Remark}

\usepackage{hyperref} % 导入 hyperref 包以支持超链接功能
\usepackage{lineno}

\usepackage{floatrow}
\usepackage{wrapfig}
\usepackage{hyperref}       % hyperlinks
\usepackage{url}            % simple URL typesetting
\usepackage{booktabs}       % professional-quality tables
\usepackage{amsfonts}       % blackboard math symbols
\usepackage{nicefrac}       % compact symbols for 1/2, etc.
\usepackage{microtype}      % microtypography
\usepackage{xcolor}         % colors
\usepackage{bm}             % some math fonts
\usepackage{mathrsfs}
\usepackage{xr}

\title{Relax and Merge: A Simple Yet Effective Framework for Solving Fair $k$-Means and $k$-sparse Wasserstein Barycenter Problems}

\date{}

\begin{document}

\author[1]{Shihong Song}
\author[1]{Guanlin Mo\thanks{The first two authors contribute equally.}}
\author[1]{Qingyuan Yang}
\author[2]{Hu Ding}
\affil[1]{University of Science and Technology of China, Hefei, China 

(\texttt{\{shihongsong, moguanlin, yangqingyuan\}@mail.ustc.edu.cn})}
\affil[2]{University of Science and Technology of China, Hefei, China 

(\texttt{huding@ustc.edu.cn})}

\maketitle

\begin{abstract}
The fairness of clustering algorithms has gained widespread attention across various areas in machine learning. In this paper, we study \emph{fair $k$-means clustering} in Euclidean space. 
  Given a dataset comprising several groups, the fairness constraint requires that each cluster should contain a proportion of points from each group within specified lower and upper bounds.
  Due to these fairness constraints, determining the locations of $k$ centers and finding the induced partition are quite challenging tasks.
  We propose a novel ``Relax and Merge'' framework that returns a $(1+4\rho + O(\epsilon))$-approximate solution, { where $\rho$ is the approximate ratio of an off-the-shelf vanilla $k$-means algorithm and $O(\epsilon)$ can be an arbitrarily small positive number}. If equipped with a PTAS of $k$-means, our solution can achieve an approximation ratio of $(5+O(\epsilon))$  with only a slight violation of the fairness constraints, {which improves the current state-of-the-art approximation guarantee.} Furthermore, using our framework, we can also obtain a $(1+4\rho +O(\epsilon))$-approximate solution for the {\em $k$-sparse Wasserstein Barycenter} problem, which is a fundamental optimization problem in the field of optimal transport, and a $(2+6\rho)$-approximate solution for the \emph{strictly fair $k$-means clustering} 
  with no violation, both of which are better than the current state-of-the-art methods. 
  In addition, the empirical results 
  demonstrate that our proposed algorithm 
 can significantly 
  outperform baseline approaches in terms of 
  clustering 
  cost.
\end{abstract}

\section{Introduction}
\label{sec-intro}
% % \vspace{-0.1in}
Clustering is one of the most fundamental problems in the area of machine learning. A wide range of practical applications rely on effective clustering algorithms, such as 
% recommendation systems \citep{burke2011recommender,das2014clustering,pham2011clustering} and 
feature engineering \citep{glassman2014feature,alelyani2018feature}, image processing~\citep{coleman1979image,chang2017deep}, and bioinformatics~\citep{ronan2016avoiding,nugent2010overview}. 
%to have a better analysis of {data}. 
In particular, the $k$-means clustering problem has been extensively studied in the past decades~\citep{DBLP:journals/prl/Jain10}. Given an input dataset $P \subset \mathbb{R}^d$, the goal of the $k$-means problem is to find a set $S$ of at most $k$ points for minimizing the clustering cost, which is the sum of the squared distances from every point of $P$ to its nearest neighbor in $S$.  
% In recent years, the study on clustering \emph{fairness} has in particular attracted a great amount of attention~\citep{chierichetti2017fair,bera2019fair,bohm2021algorithms,huang2019coresets,chen2019proportionally,ghadiri2021socially}. 
% The concerns on fairness are motivated by various fields like education, social security, and cultural communication. (citep something)
{In recent years, motivated by various fields like education, social security, and cultural communication, the study on \emph{fairness} of clustering has in particular attracted a great amount of attention~\citep{chierichetti2017fair,bera2019fair,huang2019coresets,chen2019proportionally,ghadiri2021socially, backurs2019scalable}. }

%{ 
%While using clustering algorithms to solve practical problems, fairness is a critical factor that we should pay attention to}.

In this paper, we consider the problem of \textbf{\emph{$(\alpha,\beta)$-fair $k$-means clustering}} that was initially proposed by \citet{chierichetti2017fair} and then generalized by \citet{bera2019fair}.
 %Under the definition of 
% {
 %$(\alpha, \beta)$ fair \citep{bera2019fair}}, 
 Informally speaking, we assume that the given dataset $P$ consists of $m$ groups {of points}, and  the ``fairness'' constraint requires that in each obtained cluster, the points from each group should take a fraction between pre-specified lower and upper bounds.
 %(the formal definition is shown in {Section~\ref{sec-pre}}). 
 %{We also consider a special but important case of this problem, named \emph{strictly fair} k-clustering.} 
%  \citet{bera2019fair} showed that a $\rho$-approximate algorithm for vanilla $k$-means clustering can also provide a $(2+\sqrt{\rho})^2$- approximate solution~\footnote{In their paper, the approximate ratio is written as $(2+\rho)$ because they consider general $k$-clustering, which has a root outside the $k$-means cost function.} for fractional fair $k$-clustering, where  the ``fractional'' means we allow split point to be assigned to different clusters with fractional weights.
%  %with constant approximation ratio. 
% To obtain an integral solution, they also proposed a rounding technique with a slight violation on the fair constraint but ensuring the same clustering cost.
\citet{bera2019fair} showed that a $\rho$-approximate algorithm for vanilla $k$-means can provide a $(2+\sqrt{\rho})^2$- approximate solution \footnote{In their paper, the approximate ratio is written as $(2+\rho)$ because they added a squared root to the $k$-means cost function.} for $(\alpha,\beta)$-fair $k$-clustering with a slight violation on the fairness constraints, where the ``violation'' is formally defined in Section~\ref{sec-pre}.
Regarding to the no violation scenario, \citet{dai2022fair} and \citet{WU2024114332} both obtained an $O(\log k)$-approximate solution for fair $k$-median, and their basic technique is dynamic programming~\citep{arora1998approximation} after embedding the metric space into a tree. \citet{WU2024114332} also achieved a quasi-polynomial-time approximate scheme in doubling metric by using dynamic programming with split tree~\citep{talwar2004bypassing}.
 Furthermore, \citet{bohm2021algorithms} studied the ``strictly'' fair $k$-means clustering problem, where it requires that the number of points from each group should be uniform  in every cluster; they obtained  a $(2+\sqrt{\rho})^2$ approximate solution without violation.
 %, where $\rho$ is the approximation ratio of an invoked vanilla $k$-means clustering algorithm, 
 {These fair $k$-means algorithms can also be accelerated by using the coreset techniques, such as~\citep{huang2019coresets,braverman2022power,DBLP:journals/jcss/BandyapadhyayFS24}.} There also exist  polynomial-time approximation scheme (PTAS) for fair $k$-means, such as the algorithms proposed in~\citep{bohm2021algorithms, schmidt2020fair,DBLP:journals/jcss/BandyapadhyayFS24}, but their methods have an exponential time complexity in $k$.
 We are also aware of several other different definitions of fairness for clustering problems, such as the \emph{proportionally fair} clustering~\citep{chen2019proportionally, micha2020proportionally} and \emph{socially fair} $k$-clustering clustering~\citep{ghadiri2021socially,abbasi2021fair,makarychev2021approximation, chlamtavc2022approximating}.
 %considered . More specifically, a k-clustering instance is proportionally fair {if there are no  more than a certain number of points in the dataset that are too far from their cluster centers.} \citet{ghadiri2021socially} proposed the \emph{socially fair} $k$-means algorithm on given groups of data points, {which ensures the cost of each group equal while minimizing the $k$-means cost.}

% However, since fair k-clustering problems usually need solutions that are integral ()
% Another problem closely related to fair $k$-means is k-sparse Wassertein Barycenter \citep{agueh2011barycenters} (k-sparse WB, the formal definition is shown in Section~\ref{sec-pre}). 

Another problem closely related to fair $k$-means is the so-called ``\textbf{$k$-sparse Wassertein Barycenter (WB)}'' \citep{agueh2011barycenters} (the formal definition is shown in Section~\ref{sec-pre}). 
The Wasserstein Barycenter is a fundamental concept in optimal transport theory, and it represents the ``average" or central distribution of a set of probability distributions. It plays a crucial role in various applications such as image processing~\citep{bonneel2015sliced,cuturi2014fast}, data analysis~\citep{rabin2012wasserstein}, and machine learning~\citep{backhoff2022bayesian,metelli2019propagating}.
%by providing a way to combine multiple distributions into a single representative distribution. 
Given $m>1$ discrete distributions, the goal of the $k$-sparse WB problem is 
% to find a distribution with $k$-sparse support (the support contains at most $k$ points in the space) that minimizes the sum of the Wasserstein distances~\citep{villani2021topics} between itself to all the given distributions. 
{to find a discrete distribution (i.e., the barycenter) that minimizes the sum of the Wasserstein distances~\citep{villani2021topics} between itself to all the given distributions, and meanwhile the support size of the barycenter is restricted to be no larger than a given integer $k\geq 1$.} 
If relaxing the ``$k$-sparse'' constraint (\emph{i.e.,} the barycenter is allowed to take a support size larger than $k$), \citet{altschuler2021wasserstein} presented an algorithm based on linear programming, which can compute the WB within fixed dimensions in polynomial time. 
{ If the locations of the WB supports are given, the problem is 
called ``fixed support WB'', which can be solved by using several existing algorithms~\citep{claici2018stochastic, cuturi2014fast, cuturi2016smoothed, lin2020fixed}.}
% { If the locations of WB are given, this problem is called ``fixed support WB'' and there are many works can address it~\citep{claici2018stochastic, cuturi2014fast, cuturi2016smoothed}}
%Furthermore, if we limit the number of the support of the barycenter to no more than $k$, the modified problem is called k-sparse Wasserstein Barycenter. It can be proved that a k-sparse WB instance can be seen as a fractional version of strictly fair $k$-means with a weighted input dataset, and 
If we keep the ``$k$-sparse'' constraint, it has been proved that the problem is  NP-hard \citep{borgwardt2021computational}. To the best of our knowledge, the current lowest approximation ratio of k-sparse WB problem is also $(2+\sqrt{\rho})^2$ (same with the aforementioned approximation factor for fair $k$-means), as recently studied by \citet{yang2024approximate}. 
% {\color{black}[Steffen Borgwardt's works] we have cited his works}
In fact, we can regard this problem as a special case of fair $k$-means clustering, where each input distribution is an individual group and the unique cost measured by ``Wasserstein distance'' is implicitly endowed with a kind of fairness.
%The given distributions can be analogized to the groups in fair clustering, 
%and} the fairness constraints are also included in WB problem. 
This observation from ~\citet{yang2024approximate} inspires us to consider solving the $k$-sparse WB problem under our framework. 

% % \vspace{-0.1in}
% designing its algorithm in the current article. 
%imply that it is possible to solve WB problem using the algorithms for solving fair clustering, and vice versa. 
{\textbf{Why fair $k$-means is so challenging?}
%As we know, the vanilla $k$-means problem has been widely studied.
% {citep something}. 
%However, it is much more challenging to analyze the case with fair constraints. For example, given the optimal solution of vanilla $k$-means, when dividing the input dataset into only 2 groups, 50\% of the data points can violate the fairness constraints. 
Though the fair $k$-means clustering has been extensively studied in recent years, their current state-of-the-art approximation qualities are still not that satisfying. The major difficulty   arises from the lack of ``locality property"~\citep{ding2020unified,bhattacharya2018faster}   caused by fair constraints. More precisely, in a clustering result of vanilla $k$-means, each client point obviously belongs to its closest center. That is, a  $k$-means clustering implicitly forms  a {\em Voronoi diagram}, where the cell centers are exactly the $k$ cluster centers, and  the client points in each Voronoi cell form a cluster.
%, and any cluster center is the centroid of the points in its Voronoi cell. 
However, when we add some fair constraints, such as requiring that the proportion of points of each group should be equal in each cluster, the situation becomes more complicated. Given a set of cluster center locations, because the groups of client points within a Voronoi cell may not be equally distributed, some points are forced to be assigned to other Voronoi cells. This loss of locality introduces significant uncertainty for the selection of cluster center positions. 
%Since it is challenging to find an appropriate way to obtain the cluster center positions while preserving the fair constraints, 
The previous works \citep{bera2019fair, bohm2021algorithms} do not pay much attention on how to handle this locality issue when searching for the cluster centers, instead, they directly apply vanilla $k$-means algorithms to the entire input dataset or a group, and use the obtained center locations as the center locations for fair $k$-means. It is easy to notice that their methods could result in a certain gap  with the optimal fair $k$-means solution. To narrow this gap, we attempt to design some more effective way to determine the center locations, where the key part that we believe, should be how to encode the fair constraints into the searching algorithm.  
%We believe that there exists a more suitable way to determine these center locations, leading to a better approximation ratio.
}

\textbf{Our key ideas and main results.}
{Our key idea relies on an important observation, where we find that the fair $k$-means problem is inherently related to a classic geometric structure, ``$\epsilon$-approximate centroid set'', which was firstly proposed by \citet{matouvsek2000approximate}. Roughly speaking, given a dataset, an $\epsilon$-approximate centroid set should contain at least one point that approximately represents the centroid location of any subset of this given dataset. It means that the $\epsilon$-approximate centroid set contains not only the approximate centroids based on the Voronoi diagram, but also the approximate centroids of those potential fairness-preserving clusters. 
%Note that any fair $k$-means center locates on the centroid of a cluster to achieve the lowest cost. 
%So if we propose a method to select the centroids of those "good" fair $k$-means clusters, then we have the proper locations of fair $k$-means centers.
}

% Our first contribution is to illustrate the relation between fair $k$-means and linear programming and vanilla $k$-means. The vanilla $k$-means problem has been widely studied{citep something}. However, it is much more challenging to analyze the case with fair constraints. For example, given the optimal solution of vanilla $k$-means, when divide the input dataset into only 2 groups, 50\% of the data points can violate the fairness constraints. Though the fair $k$-means clustering has been extensively studied in recent years, their current state-of-the-art approximation qualities are still not that satisfying. In this paper, we show that given any $\epsilon>0$, one can achieve a $(1+4\rho + \epsilon)$ approximate algorithm for the  $(\alpha, \beta)$-fair $k$-means problem in Euclidean space.

Inspired by the above observation, we illustrate the relationship between fair $k$-means and $\epsilon$-approximate centroid set first, and then propose a novel \emph{Relax-and-Merge} framework. In this framework, we relax the constraints on the number of clusters $k$;  we focus on utilizing fair constraints to cluster the data into small and fair clusters, which are then merged together to determine the positions of $k$ cluster centers. 
As shown in Table~\ref{Comparison on approximation ratio}, our result is better than the state of the art works~\citep{bera2019fair, bohm2021algorithms}.
%that give $(2+\sqrt{\rho})^2$ approximation factor when $\rho$ is small enough. 
Equipped with a PTAS for $k$-means problem (e.g., the algorithm from \citet{cohen2019local}), our algorithm yields a ${5+O(\epsilon)}$ approximation factor. 
%,where $\rho$ is the approximation ratio of the vanilla $k$-means algorithm. 
% First, we eliminate the constraint that the number of the clusters cannot exceed $k$ and obtain a relaxed solution through a linear programming process. Then we run a vanilla $k$-means algorithm on it to yield the locations of $k$ centers. Finally we assign all data points to their closest centers. 
% We refer to the method we use as "Relax and Merge"
% Besides, by using the rounding technique proposed by \citep{bera2019fair}, we can transform our fractional solution into an integral solution with a slight violation of the fair constraints. 
%We also provide an alternative approach by using coreset technique \citep{braverman2022power} to reduce the time complexity. 
%Our main results are as follows:
%\textbf{(1)} Our first contribution is to illustrate the relationship between fair $k$-means and $\epsilon$ centroid set and vanilla $k$-means. We show that given any $\epsilon>0$, one can achieve a $(1+4\rho + \epsilon)$ approximate algorithm for the  $(\alpha, \beta)$-fair $k$-means problem in Euclidean space.
We also present two important extensions from our work. 
%Our method can be extended to address Wasserstein Barycenter problem. To be precise, our algorithm can obtain an $(1+\epsilon)$ approximate solution of Wasserstein Barycenter in a more direct way (the time complexity of the algorithm for solving exact WB \citep{altschuler2021wasserstein} is high) and  
The first extension is an $(1+4\rho + O(\epsilon))$ solution for k-sparse Wasserstein Barycenter.
The second one is about strictly fair $k$-means. 
%Our $(1+4\rho + \epsilon)$ approximate algorithm still works. Furthermore, we give a novel rounding method inspired by \citep{shmoys1993approximation}, and our method can transform a fractional solution to integral with only at most ``$1$'' violation of the fair constraints.
We give a refined algorithm of \emph{Relax and Merge} that yields a no-violation solution with a $(2+6\rho)$ approximation factor, which is  better than the state of the art work~\citep{bohm2021algorithms}.

\begin{table}[h]
\centering
\caption{{Comparison of the approximation ratios for fair $k$-means and $k$-sparse WB. The ``general case'' includes $(\alpha,\beta)$-fair $k$-means, strictly $(\alpha,\beta)$-fair $k$-means and $k$-sparse WB.}
%The algorithm of \citep{schmidt2020fair} requires the input consisting of only 2 groups and the algorithm of \citep{bohm2021algorithms} and our algorithm~\ref{alg-balance} only works for strictly fair case and do not violate any fairness constraint
}

    \label{Comparison on approximation ratio}
    \begin{tabular}{ccccc}
    \toprule
        \textbf{Algorithms}  & \begin{tabular}{c}
        \textbf{Approximation} \\ 
        \textbf{ratio}
        \end{tabular} & \begin{tabular}{c}
        \textbf{When} \\ 
        \textbf{$\rho = 1+O(\epsilon)$}
        \end{tabular}
        &
        \textbf{Note on the quality}
        \\ \cmidrule(r){1-4}
        \citet{bera2019fair}   & $(2+\sqrt{\rho})^2$ & $9+O(\epsilon)$ & general case \\
        \citet{schmidt2020fair}   & $5.5\rho + 1$ & $6.5+O(\epsilon)$ & two groups only\\
       \citet{bohm2021algorithms}   &  $(2+\sqrt{\rho})^2$ & $9+O(\epsilon)$ & strictly only, no violation\\ 
       \citet{yang2024approximate} & $(2+\sqrt{\rho})^2$ & $9+O(\epsilon)$ & $k$-sparse WB \\
        Algorithm~\ref{alg-fair}, now  & $1+4\rho+O(\epsilon)$ & $5+O(\epsilon)$ & general case \\
        Algorithm~\ref{alg-balance}, now  & $2 + 6\rho$ & $8+O(\epsilon)$ & strictly only, no violation\\
\bottomrule 
    \end{tabular}
\end{table}
% \vspace{-0.2in}
\textbf{Other Related Works on $k$-Means}
{The vanilla $k$-means problem is a topic that has been widely studied in both theory and practice.} It has been proved that $k$-means clustering is NP-hard even in $2D$ if $k$ is large \citep{DBLP:journals/tcs/MahajanNV12}. In high dimensions, even if $k$ is fixed, say $k=2$, the $k$-means problem is still NP-hard \citep{DBLP:journals/ml/DrineasFKVV04}. Furthermore, \citet{lee2017improved} proved the APX-hardness result for Euclidean $k$-means problem, which implies that it is impossible to approximate the optimal solution of $k$-means below a factor 1.0013 in polynomial time under the assumption of P $\ne$ NP. Therefore, a number of approximation algorithms have been proposed in theory. 
If the dimension $d$ is fixed, \citet{kanungo2002local} obtained a ($9+O(\epsilon)$)-approximate solution
%of the $k$-means problem in Euclidean space 
by using the local search technique. Roughly speaking, the idea of local search is swapping a small number of points in every iteration, so as to incrementally improve the solution until converging at some local optimum. Following this idea, \citet{cohen2019local} and \citet{friggstad2019local} proposed the PTAS for $k$-means in low dimensional  space. For high-dimensional case with constant $k$, \citet{DBLP:journals/jacm/KumarSS10} proposed an elegant peeling algorithm that iteratively finds the $k$ cluster centers and eventually obtain the PTAS. 

% \textbf{Wasserstein distance and Optimal Transport.} Optimal Transport (OT) is the one of the most fundamental problem in machine learning. Wasserstein distance measures the optimum of OT (the formal definition is shown in Section~\ref{sec-pre}).
% A direct way to compute Wasserstein distance is to reduce it to min-cost flow problem, which can be solved by network simplex algorithm~\citep{altschuler2021wasserstein, chen2022maximum, khesin2019preconditioning, lee2014path}. \citet{cuturi2013sinkhorn} propose ``Sinkhorn distance'' as a variant of Wasserstein distance in order to build faster algorithm. After that, many works focus on improve Cuturi’swork. \citet{altschuler2017near} propose a nearly linear time algorithm. And \citet{genevay2016stochastic} propose a stochastic algorithm to address large-scale optimal transport.
\vspace{-0.1in}
\section{Preliminaries}
\label{sec-pre}
 \vspace{-0.1in}
\textbf{Notations.} In this paper, we always assume that the dimensionality $d$ of the Euclidean space is constant. Let $P$ denote the set of $n$ client points located in Euclidean space $\mathbb{R}^d$.
% and each point $p\in P$ has a non-negative weight that is denoted by $w(p)$ 
% {(for simplicity, we assume that $w(p)=1$ for each $p\in P$; {\color{black}for the general case, we can normalize all the weights with setting $\min_{p\in P}w(p)=1$, and each point can be regarded as a set of $w(p)$ overlapping points with unit weight}).} 
The set $P$ consists of $m$ different  groups (not necessarily disjoint), \emph{i.e.}, $\boldsymbol{P=\cup_{i=1}^{m}P^{(i)}}$, and each group has the size $\boldsymbol{|P^{(i)}| = n^{(i)}}$  (we use the superscript ``$(i)$''
 to denote the group's index). The Euclidean distance between two points $a,b \in \mathbb{R}^d$ is denoted by $\Vert a-b\Vert $; the distance between a point $a$ and any set $Q\subset\mathbb{R}^d$ is denoted by $\mathtt{dist}(a,Q) = \min_{q\in Q} \Vert a-q\Vert $, and the nearest neighbor of $a$ in $Q$ is denoted as $\mathcal{N}(a, Q)$. 
 %Given any set $E\in \mathbb{R}^d$, 
 The centroid of a set $Q$ is denoted by $\mathtt{Cen}(Q)$. 
 %If the point set $P$ is weighted, for every $p\in P$, we use $w(p)$ to denote the weight of point $p$. 

For the vanilla $k$-means problem, the client points are always assigned to their nearest center. However, if considering the fairness constraint, the assignment may not be that straightforward. To describe the fair $k$-means clustering more clearly, we introduce the ``\textbf{assignment matrix}'' first. Given any candidate set of $k$ cluster centers $S$, we define the assignment matrix 
%context of $(\alpha, \beta)$ fair $k$-means, client points could be assigned to any cluster. Therefore, we use a matrix 
$\boldsymbol{\phi_S}:P\times S \rightarrow \mathbb{R}^+$ to indicate the assignment {relation between the client points and cluster centers}.
For every $p \in P$ and $s\in S$, $\phi_S(p,s)$ denotes the proportion that
% of $w(p)$ that 
is assigned to center $s$ (e.g., we may respectively assign $30\%$  and $70\%$ to two different centers). Obviously, we have $\sum_{s\in S} \phi_S(p,s) = 1$.
%because the total weight of every client point is $1$. 
For each center $s\in S$, we use $w(s) = \sum_{p\in P}\phi_S(p,s)$ to denote the amount of weight assigned to   $s$; for each group $P^{(i)}$, we similarly define the function $w^{(i)}(s) = \sum_{p\in P^{(i)}} \phi_S(p,s)$. Let $\mathtt{Cost}(P, S, \phi_S)$ denote the cost of input instance $P$ with given $S$ and $\phi_S$:
\begin{eqnarray}
\mathtt{Cost}(P, S, \phi_S)=\sum_{p\in P} \sum_{s\in S} \Vert p-s\Vert ^2\phi_S(p,s).\label{for-cost}
\end{eqnarray}

%received from all the client points and we use $w^{(i)}_S(s) = \sum_{p\in P^{(i)}} \phi_S(p,s)$ to denote how many weights does $s$ received from client points of the $i$-th group.}

%The next definition of $(\alpha, \beta)$-fair $k$-means problem comes from \citet{bera2019fair}.
\begin{problem}[$(\alpha, \beta)$-fair $k$-means clustering \citep{bera2019fair}]
\label{pro-fair-means}
  Given an instance $P$ as described above
%  =\cup_{i=1}^m P^{(i)} \subset \mathbb{R}^d$ with $|P|=n$, where $P^{(i)}$ denote the point group with color $i$ (not necessary disjoint), 
  and two parameter vectors $\alpha,\beta \in [0,1]^m$, the goal of the \textbf{$\boldsymbol{(\alpha, \beta)}$-fair $\boldsymbol{k}$-means} clustering is to find the set $S$ consisting of $k$ points and an assignment matrix $\phi_S$,  such that the clustering cost (\ref{for-cost}) is minimized,
 % , where $\phi_S$ is the assignment matrix. 
  and meanwhile  each cluster center $s\in S$ should satisfy the fairness constraint: $\beta_i w(s)\le w^{(i)}(s) \le \alpha_i w(s)$ for every  
 % the fairness constraint $\beta_i \sum_{p\in P}\phi_S(p,s)\le \sum_{p\in P^{(i)}}\phi_S(p,s) \le \alpha_i \sum_{p\in P}\phi_S(p,s)$ should be satisfied for each 
  $i\in \{1, 2, \cdots, m\}$. Here, we use $\alpha_i, \beta_i$ to denote the $i$-th entry of  $\alpha$ and $\beta$, respectively.
  
  {Moreover, if the $m$ groups are disjoint with equal size (i.e., $n^{(i)}=n/m$ for any $i$), and $\alpha_i=\beta_i=1/m$ for each  group $P^{(i)}$, we say this is a \textbf{strictly $\boldsymbol{(\alpha, \beta)}$-fair $\boldsymbol{k}$-means} clustering problem. }
  %and use $[m]$ to denote $\{1,2,\cdots,m\}$ for simplicity.
  
  % \begin{equation}
  %   \begin{aligned}
  %     \min_{S, \phi} \quad     & \sum_{p\in P} \Vert p-\phi(p)\Vert ^2                                       \\
  %     s.t.         \quad & |\{p\in P^{(i)}|\phi(p)=f\}|\le \alpha_i |\{p\in P|\phi(p)=f\}| \quad \forall f\in S, \forall i\in [m], \\
  %      \quad & |\{p\in P^{(i)}|\phi(p)=f\}|\ge \beta_i |\{p\in P|\phi(p)=f\}| \quad \forall f\in S, \forall i\in [m], 
  %   \end{aligned}
  % \end{equation}
  % where $\phi$ is the assignment from $P$ to $S$, \emph{i.e.}, $\phi:P\times S \rightarrow \mathbb{R}_+$.

  % {where $w^{(i)}(f)$ denotes the number of points that belong to $P^{(i)}$ and being clustered to center $f$. May need to distinguish clusters and centers of a solution? } For given $P$, $S$, $\phi$ and $w$, the value of $\sum_{p\in P} \sum_{f\in S} \phi(p,f)\Vert p-f\Vert ^2$ is denoted by $\mathtt{Cost}(P, S, \phi, w)$.
\end{problem}

%If the given groups are disjoint and $\alpha_i = \beta_i = n/m$, we call this special case \emph{strictly fair k-means} problem. 

For Problem~\ref{pro-fair-means}, we can specify two types of solutions: \textbf{fractional} and \textbf{integral}. Their difference is only from the restriction on the assignment matrix $\phi_S$. For the first one, each entry $\phi_S(p,s)$ can be any real number between $0$ and $1$; but for the latter one, we require that the value of $\phi_S(p,s)$ should be either $0$ or $1$, that is, the whole weight of $p$ should be assigned to only one cluster center. 

% Our final goal is to obtain a integral solution via rounding a fractional solution into integral.
How to round a fractional solution into integral while preserving fairness constraints is still an open problem. \citet{bera2019fair} introduced the \textbf{violation factor} to measure the violations of fairness constraints after rounding: an assignment matrix $\phi_S$ is a $\lambda$-violation solution if $\beta_i \sum_{p\in P}\phi_S(p,s)  - \lambda \le \sum_{p\in P^{(i)}} \phi_S(p,s) \le \alpha_i \sum_{p\in P}\phi_S(p,s) + \lambda, \quad \forall s\in S, \forall i\in [m]$.
In their paper, a fractional solution can always be rounded to integral, but it introduces some violations , which will be discussed in Section~\ref{sub-frac}.
In this paper, we use $OPT$ to denote the optimal integral cost of Problem~\ref{pro-fair-means}. We use $S_{\mathtt{opt}}=\{\tilde{s}_1, \tilde{s}_2, \cdots, \tilde{s}_k\}$ to denote the optimal solution of integral fair $k$-means problem and its assignment matrix is denoted by $\phi_{S_{\mathtt{opt}}}$. 
For each $\tilde{s}_j$, let $C_j=\{p\in P\mid \phi_{S_{\mathtt{opt}}}(p,\tilde{s}_j) > 0\}$ be the corresponding cluster, i.e., the set of point assigned to it.
%, \emph{i.e.} $\phi_{S_{opt}}(p,\tilde{s}) > 0$. 
% Similarly, we define $C_t$ as the set of points assigned to $t\in T$. 
% Let { $\pi(t) = \frac{\sum_{p\in P}p\cdot \phi_T(p,t)}{\sum_{i=1}^{m} w^{(i)}(t)}$}. 
% It is easy to see that $\pi(t)$ is the centroid of $C_t$. 
%For each point $p\in P$, we use $nrst_O(\cdot)$ to denote the nearest $\tilde{s} \in S_{opt}$.
% If the fair {constraints are} modified as $\beta_i \sum_{i=1}^mw^{(i)}(f)- \lambda \le w^{(i)}(f)  \le \alpha_i \sum_{i=1}^mw^{(i)}(f) + \lambda, \quad \forall f\in S, \forall i\in [m]$, the relaxed problem is called fair $k$-means with $\lambda$ violation. 
A simple observation is that, if given a fixed candidate cluster centers set $S$, the assignment matrix 
$\phi_S$ can be obtained via solving a linear programming (we can view the $n\times k$ entries of $\phi_S$ as the variables):
  \begin{equation}
    \begin{aligned}
     \min_{\phi_S} \quad     & \sum_{p\in P} \sum_{s\in S} \Vert p-s\Vert ^2  \phi_S(p,s)                                     \\
      s.t.         \quad & \beta_i \sum_{p\in P}\phi_S(p,s) \le \sum_{p\in P^{(i)}} \phi_S(p,s) \le \alpha_i \sum_{p\in P}\phi_S(p,s), \quad \forall s\in S, \forall i\in [m], \\
      % & \sum_{p\in P^{(i)}} \phi(p,f)=w^{(i)}(f), \quad \forall f\in S , \forall i\in [m],\\
      & \sum_{s\in S} \phi_S(p,s) = 1, \quad \forall p\in P.
      % & \sum_{f\in S} w^{(i)}(f) = n^{(i)}, \quad \forall i\in [m], \\
      % &\phi(p,f) \in \{0,1\}, \quad \forall p\in P, f\in S.
      \label{eq-fair}
    \end{aligned}
  \end{equation}
%with the lowest cost is called \emph{fair assignment} problem, which can be written as an integral linear programming. If we allow the elements of $\phi_S$ to be fractional, the relaxed assignment problem is a linear programming (LP) as follows. 
If we want to compute an integral solution, the above (\ref{eq-fair}) should be an integer LP. 
Given a set $S$, $\boldsymbol{\phi_S^*}$ denotes the optimal solution of (\ref{eq-fair}) and $\boldsymbol{\tilde{\phi}_S}$ denotes the corresponding optimal integral solution. 

The following proposition is a folklore result that has been used in many articles on clustering algorithms (e.g., \citep{kanungo2002local}). We will also repeatedly use it in our proofs.

\begin{proposition}
Given a finite weighted point set $Q\subset \mathbb{R}^d$, 
%if $\mathtt{cen}(Q)$ is the centroid of $P$, then 
for any point $a$, 
$\sum_{q\in Q}w(q)\Vert a - q \Vert^2  = \sum_{q\in Q}w(q) \Vert q-\mathtt{Cen}(Q)\Vert ^2  + w(Q)\cdot \Vert a-\mathtt{Cen}(Q)\Vert ^2$, where $w(Q)$ is the total weight of $Q$.
\label{lem-centroid}
\end{proposition}

%\begin{proposition}[Squared Triangle Ineuqality]
%    $\Vert a - b \Vert^2 \le 2(\Vert a - c \Vert^2 + \Vert c - b \Vert^2)$.
%\end{proposition}

% Our proposed framework also relies on an important geometric structure, 
Next we introduce  an important geometric structure ``$\epsilon$-approximate centroid set'', which was 
firstly proposed by \citet{matouvsek2000approximate}. 
Roughly speaking, the $\epsilon$-approximate centroid set approximately covers the centroids of any subset of given data, even though the subsets do not align with the ``Voronoi diagram'' structure (as discussed in Section~\ref{sec-intro}). 
% Roughly speaking, an $\epsilon$-approximate centroid set 
% should contain at least one point that approximately represents the centroid location of any subset of given data. 
% This property can help us to find the centroids of potential fair clusters. Formally, we have the following definition.

\begin{definition}
Given a finite set $P\subset \mathbb{R}^d$ and a small parameter $\epsilon>0$, we use $\mathtt{CS}_\epsilon(P)$ to denote an \textbf{$\epsilon$-approximate centroid set} of $P$ that satisfies: for any nonempty subset $Q\subseteq P$, there always exists a point $v\in \mathtt{CS}_\epsilon(P)$ such that $\Vert v - \mathtt{Cen}(Q) \Vert \le \frac{\epsilon}{3}\sqrt{\frac{1}{|Q|}\sum_{q\in Q}\Vert q-\mathtt{Cen}(Q) \Vert^2}$.
\label{def-centroid-set}
\end{definition}
 
\begin{remark}
\label{rem-complexity-matouvsek}
    \citet{matouvsek2000approximate} also presented a construction algorithm based on the space partitioning technique ``quadtree''~\citep{DBLP:journals/acta/FinkelB74}. In Appendix~\ref{sec-centroid-set}, we briefly illustrate the role of the $\epsilon$-approximate centroid set in preserving fairness constraints and how to construct it.
    % {\color{black} First, we use a quadtree to partition the space into hierarchical cubes. At each level of the tree, we construct a grid. The length of the grid is set to ensure that the grid points can always cover all approximate centroids of all cubes at this level. The approximate centroid set is the union of all grid points across all levels.} 
    The size of the obtained $\epsilon$-approximate centroid set is $O(|P|\epsilon^{-d}\log (1/\epsilon))$ and the construction time complexity is $O(|P|\log |P| + |P|\epsilon^{-d}\log (1/\epsilon))$. 
    % We briefly introduce the idea of construction an $\epsilon$-approximate centroid set in \ref{sec-centroid-set}.
    % {{**may add some figure/explaination on it in supplement.**}}
\end{remark}

Next, we give the formal definition of \textbf{$k$-sparse Wasserstein Barycenter} problem.

\begin{definition}[Wasserstein Distance]
    Let $P$ and $Q$ be weighted point sets supported in $\mathbb{R}^d$. Wasserstein distance is the minimum transportation cost between $P$ and $Q$:
    $\mathcal{W}(P,Q) = \min_F \sqrt{\sum_{p\in P} \sum_{q\in Q} F(p,q) \Vert p - q\Vert ^2}$,
    where the transport matrix $F: P \times Q \rightarrow [0,1]$ should satisfy: $\sum_{p\in P} F(p, q) = w(q)$ for any $q\in Q$, and $\sum_{q\in Q} F(p,q) = w(p)$ for any $p\in P$.
  \end{definition}

 For a weighted set $S$, we use $\mathtt{supp}(S)$ to denote its support, i.e., the set that shares the same location of $S$ but not weighted. 
The number of points is $\mathtt{supp}(S)$ is denoted by $|\mathtt{supp}(S)|$.
\begin{problem}[$k$-sparse Wassertein Barycenter ($k$-sparse WB)]
    Given $m$ discrete probability distributions $P^{(1)}, \cdots, P^{(m)}$ supported on $\mathbb{R}^d$, \emph{WB} is the probability distribution $S$ minimizing the sum of squared Wasserstein distances to them, i.e.,
        $\arg \min_{S} \sum_{i=1}^m \mathcal{W}^2(P^{(i)}, S)$.
        The problem is called \textbf{$k$-sparse Wasserstein Barycenter} if we restrict $|\mathtt{supp}(S)|\le k$
    %where $\mathcal{W}(\cdot, \cdot)$ denotes the Wassertein distance between two distributions.
    \label{prob-wb}
\end{problem}
% \vspace{-0.1in}

%Problem \ref{prob-wb} is called \textbf{$k$-sparse Wasserstein Barycenter} if we restrict $|\mathtt{supp}(S)|\le k$. 
{ In Section~\ref{sec-kwb}, we explain why this problem can be regarded as a fair $k$-means clustering.}

\section{Our ``Relax and Merge'' Framework}
\label{sec-frac}
% \vspace{-0.1in}
In general, there are two stages in clustering with fair constraints. The first stage is to find the proper locations of clustering centers, and the second stage is to assign all the client points to the centers by solving LP (\ref{eq-fair}). The previous approaches often use the vanilla $k$-means in the first stage to obtain the location of centers, and then take the fairness into account in the second stage~\citep{bera2019fair, bohm2021algorithms}. In our proposed framework, we aim to shift the consideration of fair constraints to the first stage, so as to achieve a lower approximation factor in the final result. The following theorem is our main result.

\begin{theorem}
    {Given an instance of  Problem~\ref{pro-fair-means}} and a $\rho$-approximate vanilla $k$-means clustering algorithm, there exists 
         an algorithm that can return a fractional $(1+4\rho + O(\epsilon))$ approximate solution for Problem~\ref{pro-fair-means}. Further,
        one can apply a rounding method to transform this fractional solution to an integral one with a constant violation factor while ensuring the cost does not increase.
    % an algorithm that can return a $(1+4\rho + O(\epsilon))$ approximate solution for Problem~\ref{pro-fair-means}  with constant-factor violation of fairness constraints,} where   $\epsilon$ can be any positive number.
    \label{the-approx}
\end{theorem}
% \begin{theorem}
%     {Given a rounding algorithm that}, if $T$ is an $\epsilon$-approximate centroid set of $P$, the cost of the solution returned by Algorithm~\ref{alg-fair} is at most $(1+4\rho + O(\epsilon)) \cdot OPT$. 
%     \label{approx-ratio}
% \end{theorem}

%The main procedure is to obtain a fractional solution of Problem~\ref{pro-fair-means}. 
The details  for computing the fractional solution are shown in Algorithm~\ref{alg-fair}.  The set $T$ in Algorithm~\ref{alg-fair}
%may include more than $k$ points. As we mentioned in the observation in section~\ref{sec-intro}, this relaxed center set 
contains the approximate centroids of all the potential clusters with preserving fair constraints.
%Note that when we first obtain the relaxed center set, 
Then we solve a linear program to obtain the relaxed solution $(T,\phi^*_T)$ that also preserves the fair constraints. Because of that, the following $k$-means procedure is able to determine the appropriate locations for the cluster centers of Problem~\ref{pro-fair-means}.

%the imposed information of fair constraints.

%which makes the final center locations better than the previous works. 

\begin{algorithm}[h]
  % \SetAlgoLined
  \caption{\sc{Fractional Fair $k$-means}}
  \label{alg-fair}
  \KwIn{The dataset $P$, $k$, $\alpha$, $\beta$, and $\epsilon>0$}
  \textbf {Relax}: Construct a relaxed solution $T$, \emph{i.e.}, an $\epsilon$-approximate centroid set, such that $\mathtt{Cost}(P, T, \phi_T^*) \le (1+O(\epsilon)) \cdot OPT$ (see Lemma~\ref{lem-app-ratio-centroid-set}).
  % (we will discuss how to obtain such a $T$ in the following analysis).
   Here, we relax the size constraint of centers to be polynomial of $n$ rather than exactly $k$, so as to  achieve a sufficiently low cost.
   % \begin{enumerate}
   
   Solve LP (\ref{eq-fair}) on $T$ to obtain the optimal assignment matrix $\phi_T^*$. $T$ and $\phi_T^*$ can be viewed as a relaxed solution for $(\alpha, \beta)$-fair $k$-means, \emph{i.e.}, the number of centers may be more than $k$, and meanwhile, the cost is bounded and the fairness constraints are also preserved.
   
   Adjust the location of $T$. For each $t\in T$, we update the location of $t$ to be the corresponding cluster centroid  $\pi(t) = \frac{\sum_{p\in P}p\cdot \phi^*_T(p,t)}{w(t)}$. The adjusted $T$ is denoted by $\pi(T)$.
   % \end{enumerate}

\textbf{Merge:} 
%Merge $T$ to obtain the supports of final solution.
Run a $\rho$-approximate $k$-means algorithm on $\pi(T)$ to obtain centers set $S$.
Then, solve LP (\ref{eq-fair}) on $S$ to obtain the optimal assignment matrix $\phi_S^*$.
  
  \Return{S and $\phi_S^*$}
\end{algorithm}

\subsection{Algorithm for \texorpdfstring{$(\alpha,\beta)$}{} Fair \texorpdfstring{$k$}{}-means Problem}
\label{sub-frac}
% \vspace{-0.1in}

% Since we allow the assignment matrix be fractional, the fair assignment problem can be solved using LP technique as (\ref{eq-fair}), if $S$, the location of clustering centers, is given. Hence, the most difficult part of fractional $(\alpha, \beta)$ fair $k$-means is how to find the location of solution $S$. The previous works use some simple techniques, \emph{e.g.}, \citep{bera2019fair} uses vanilla $k$-means on $P$ and \citep{bohm2021algorithms} runs vanilla $k$-means on every group then choose the best one as the location of $S$. So a natural question is: \emph{is there a smarter way to find the location of the solution $S$?}

In this section, we mainly focus on the fractional version of $(\alpha, \beta)$-fair $k$-means problem. More precisely, we allow the value of the assignment function $\phi_S$ to be a real number in $[0,1]$ rather than $\{0,1\}$. To prove Theorem~\ref{the-approx}, we need the following lemmas first.
Specifically, Lemma~\ref{the-ratio} provides the bound for the cost from the merged solution $S$; Lemma~\ref{lem-app-ratio-centroid-set} shows that the $\epsilon$-approximate centroid set provides a satisfied relaxed solution with a cost  no more than $(1+O(\epsilon))OPT$. Combining with the rounding methods, Theorem~\ref{the-approx} can be obtained.

\begin{lemma}
Let $\eta$ be any positive number.    If we suppose $\mathtt{Cost}(P, T, \phi_T^*) \le \eta \cdot OPT$, then the solution $(S, \phi_S^*)$ returned by Algorithm~\ref{alg-fair} is an $\big(\eta+(2\eta + 2)\rho\big)$-approximate solution for Problem \ref{pro-fair-means}.
    \label{the-ratio}
\end{lemma}
%% \vspace{-0.1in}
\begin{proof}
% \begin{observation}
%     Every $\tilde{s} \in S_{opt}$ is the centroid of $N_O(V)$, \emph{i.e.}, 
% %$o = \frac{\sum_{p\in N_O(V)}p\cdot \phi_{S_{opt}}(p,\tilde{s})}{\sum_{p\in N_O(V)}}\phi_{S_{opt}}(p)$    
%     { $o = \frac{\sum_{p\in N_O(V)}p\cdot \phi_{S_{opt}}(p,\tilde{s})}{\sum_{i=1}^{m} w_O^{(i)}(V)}$}.
% \end{observation}
According to the definition of fractional fair $k$-means problem, the cost can be written as 

\begin{equation}
    \begin{aligned}
        \mathtt{Cost}(P, S, \phi_S^*)
        &= \sum_{p\in P}\sum_{s\in S} \Vert p-s\Vert ^2 \phi_S^*(p,s). \\
        \end{aligned}
        \end{equation}
        
Now we consider another assignment strategy: we firstly assign $P$ to $T$ according to $\phi_T^*$ ( recall that $\phi_T^*$ is the optimal fractional assignment matrix from $P$ to $T$), and then we assign every weighted point in $T$ to some $s\in S$ such that $s$ is closest point to $\pi(t)$. Since $\phi_S^*$ is the optimal assignment matrix from $P$ to $S$, the cost of this assignment strategy should have:
\begin{equation}
    \begin{aligned}
       % &\le 
        % \sum_{i=1}^{m}\sum_{p\in P^{(i)}}\sum_{t\in T} \Vert p-\mathcal{N}(\pi(t), S)\Vert ^2 \phi_T^*(p,t) \\
        % \sum_{p\in P}\sum_{t\in T} \Vert p-\mathcal{N}(\pi(t), S)\Vert ^2\phi_T^*(p,t)&\geq \mathtt{Cost}(P, S, \phi_S^*).\\
   \sum_{p\in P}\sum_{t\in T} \Vert p-\mathcal{N}(\pi(t), S)\Vert ^2\phi_T^*(p,t)&\geq    \mathtt{Cost}(P, S, \phi_S^*).\label{for-lem1-1}
    \end{aligned}
\end{equation}
Since $\pi(t)$ is the centroid of the weighted points assigned to $t$, according to Proposition~\ref{lem-centroid}, we know the left-hand side of (\ref{for-lem1-1}) should have the upper bound
        \begin{equation}
    \begin{aligned}
        &\sum_{t\in T}\Big[\sum_{p\in P} \Vert p-\pi(t)\Vert ^2\phi_T^*(p,t) + \Vert \pi(t) - \mathcal{N}(\pi(t), S)\Vert ^2w(t)\Big]\\
        &= \underbrace{\sum_{p\in P}\sum_{t\in T} \Vert p-\pi(t)\Vert ^2\phi_T^*(p,t)}_{\text{(a)}} + \underbrace{\sum_{t\in T}\Vert \pi(t) - \mathcal{N}(\pi(t), S)\Vert ^2w(t)}_{(b)}.
    \end{aligned}
\end{equation}
Then we bound (a) and (b) separately.
\begin{equation}
    \begin{aligned}
        (a)=\sum_{p\in P}\sum_{t\in T} \Vert p-\pi(t)\Vert ^2\phi_T^*(p,t) \le \sum_{p\in P}\sum_{t\in T} \Vert p-t\Vert ^2\phi_T^*(p,t) \le \eta \cdot OPT.
    \end{aligned}
\end{equation}
% The first inequality holds because $\pi(t)$ is the centroid of the weighted points assigned to $t$, so that $\pi(t)$ minimizes the sum of the squared distances. 
The first inequality holds because $\pi(t)$ is the centroid of the weighted points assigned to $t$, minimizing the weighted sum of the squared distances between them.
% from points in $C_t$. 
The second inequality holds because $\mathtt{Cost}(P, T, \phi_T^*) \le \eta \cdot OPT$. 
Next, we focus on (b). 
Suppose $S_{means}$ is the optimal $k$-means solution of $T$. Then we have:
% we use $nrst_{S^*} (t)$ denote the nearest center of $t$ in $S^*$.
\begin{equation}
    \begin{aligned}
        (b)&=\sum_{t\in T}\Vert \pi(t) - \mathcal{N}(\pi(t), S)\Vert ^2w(t) \le \rho \sum_{t\in T}\Vert \pi(t) - \mathcal{N}(\pi(t), S_{means})\Vert ^2w(t) \\
        &= \rho \sum_{p\in P}\sum_{t\in T} \Vert \pi(t) - \mathcal{N}(\pi(t), S_{means})\Vert ^2 \phi_T^*(p,t)\\
        &= \rho \sum_{p\in P}\sum_{t\in T} \big[\sum_{\tilde{s} \in S_{opt}} \Vert \pi(t) - \mathcal{N}(\pi(t), S_{means})\Vert ^2\phi_{S_{opt}}^*(p,\tilde{s})\big]\phi_T^*(p,t)\\
        &\le \rho \sum_{p\in P}\sum_{t\in T} \big[\sum_{\tilde{s} \in S_{opt}} \Vert \pi(t) - \tilde{s}\Vert ^2\phi_{S_{opt}}^*(p,\tilde{s})\big]\phi_T^*(p,t).
\end{aligned}
\end{equation}
Further, according to squared triangle inequality, we have

\begin{equation}
    \begin{aligned}
        (b)&\le \rho \sum_{p\in P}\sum_{t\in T} \Big[\sum_{\tilde{s} \in S_{opt}} \big[ \Vert \pi(t) - p\Vert  + \Vert p - \tilde{s}\Vert  \big]^2\phi_{S_{opt}}^*(p,\tilde{s})\Big]\phi_T^*(p,t)\\
        &\le \rho \sum_{p\in P}\sum_{t\in T} \sum_{\tilde{s} \in S_{opt}}2\Vert \pi(t) - p\Vert ^2\phi_{S_{opt}}^*(p,\tilde{s})\phi_T^*(p,t) \\&+ \rho\sum_{p\in P}\sum_{t\in T} \sum_{\tilde{s} \in S_{opt}}2\Vert p - \tilde{s}\Vert ^2\phi_{S_{opt}}^*(p,\tilde{s})\phi_T^*(p,t) \\
        & = 2\rho \sum_{p\in P}\sum_{t\in T} \Vert \pi(t) - p\Vert ^2\phi_T^*(p,t) + 2\rho\sum_{p\in P}\sum_{\tilde{s}\in S_{opt}} \Vert p - \tilde{s}\Vert ^2\phi_{S_{opt}}^*(p,\tilde{s}). 
        % \le (2\eta + 2)\rho \cdot OPT.
    \end{aligned}
\end{equation}
The last equality holds because for any $p\in P$, $\sum_{\tilde{s} \in S_{opt}} \phi_{S_{opt}}^*(p,\tilde{s}) = 1$.
The first term is exactly $2\rho$ times of (a) and the second term equals $2\rho \cdot OPT$.
Through combining (a) and (b),  we can obtain an approximation factor of $\eta + (2\eta + 2)\rho$.
\end{proof}
% \vspace{-0.1in}
Algorithm~\ref{alg-fair} reduces the fair $k$-means problem to computing the  set $T$. The following lemma shows that an $\epsilon$-approximate centroid set is a good candidate for $T$.
%with $\mathtt{Cost}(P, T, \phi_T^*) \le (1+\epsilon) OPT$, even though the clusters formed by $\tilde{\phi}_{S_{opt}}$ do not obey Voronoi diagram.
% \begin{lemma}
%     % Given a point set $P\subset \mathbb{R}^d$, $k\ge 2$, and $T$ be an $\epsilon$-approximate centroid set of $P$. Then for any facility set $S$ and its corresponding integral assignment matrix $\phi_S$, there always exists $T' = \{t_i \in T | i=1,\cdots, k\}$ such that
%     % $\mathtt{Cost}(P, T', \phi_T) \le (1+\epsilon) \mathtt{Cost}(P, S, \phi_S)$.
%     Given a point set $P\subset \mathbb{R}^d$, $k\ge 2$, and $T$ be an $\epsilon$-approximate centroid set of $P$. Then for optimal center set $S_{opt}$ and its corresponding integral assignment matrix $\tilde{\phi}_{S_{opt}}$, there always exists $T' = \{t_i \in T | i=1,\cdots, k\}$ such that
%     $\mathtt{Cost}(P, T', \phi^*_{T'}) \le (1+\epsilon) OPT$.
%     \label{pro-centroid}
% \end{lemma}
% Lemma \ref{pro-centroid} shows that we can find "good" enough solution for $k$-means problem from $\epsilon$-approximate centroid set. 
%And we have the following lemma.
\begin{lemma}
    If $T$ is an $\epsilon$-approximate centroid set of $P$, then {$\mathtt{Cost}(P, T, \phi_T^*) \le (1+O(\epsilon)) OPT$}.
    \label{lem-app-ratio-centroid-set}
\end{lemma}
% \vspace{-0.1in}
\begin{proof}
According to Definition~\ref{def-centroid-set}, let $t_i \in T$ denote the point such that $\Vert t_i - \mathtt{Cen}(C_i) \Vert \le \frac{\epsilon}{3}\sqrt{\frac{1}{|C_i|}\sum_{p\in C_i}\Vert p-\mathtt{Cen}(C_i) \Vert^2}$. Let $T' = \{t_1, \cdots, t_k \}$. 
%Since $T$ is the $\epsilon$-approximate centroid set of $P$, $T'$ always exists according to the .
    A key observation is that each optimal center $\tilde{s}_i$ is always the centroid of $C_i$, \emph{i.e.}, $\mathtt{cen}(C_i) = \tilde{s}_i$, so we have $\Vert t_i - \mathtt{Cen}(C_i) \Vert^2 \le \frac{\epsilon^2}{9|C_i|}\sum_{p\in C_i}\Vert p-\tilde{s}_i \Vert^2 = \frac{\epsilon^2}{9|C_i|}OPT_i$, where $OPT_i = \sum_{p\in C_i}\Vert p - \tilde{s}_i \Vert^2$.

    If we assign all points of $C_i$ to $t_i$, the cost of every $C_i$ can be written as $\sum_{p\in C_i} \Vert t_i - p \Vert^2       =$
    \begin{eqnarray}
        \begin{aligned}
            &\sum_{p\in C_i} \Vert t_i - \tilde{s}_i \Vert^2 + \sum_{p\in C_i} \Vert p-\tilde{s}_i \Vert^2         \le \frac{\epsilon^2}{9}OPT_i + OPT_i        = (1+O(\epsilon))OPT_i.
        \end{aligned}
    \end{eqnarray}
    The first equality holds due to Proposition~\ref{lem-centroid}. Since $\phi^*_{T'}$ is the optimal assignment matrix of $T'$, $\mathtt{Cost}(P, T', \phi^*_{T'}) \le  \sum_{i=1}^{k}\sum_{p\in C_i} \Vert t_i - p \Vert^2 \le (1+O(\epsilon))\sum_{i=1}^kOPT_i\le (1+O(\epsilon))OPT$.
    Finally, since $T'$ is a subset of $T$, we have $\mathtt{Cost}(P, T, \phi^*_T) \le \mathtt{Cost}(P, T', \phi^*_{T'}) \le (1+O(\epsilon)) OPT$.
%     Since $T$ is the $\epsilon$-approximate centroid set of $P$, such a set of $t_i$ can always be found in $T$.
%     A key observation is that every $\tilde{s} \in S_{opt}$ is the centroid of $N_O(V)$, \emph{i.e.}, 
% %$o = \frac{\sum_{p\in N_O(V)}p\cdot \phi_{S_{opt}}(p,\tilde{s})}{\sum_{p\in N_O(V)}}\phi_{S_{opt}}(p)$    
%     { $o = \frac{\sum_{p\in N_O(V)}p\cdot \phi_{S_{opt}}(p,\tilde{s})}{w_O(V)}$} since the centroid minimized the sum of the squared distances. Hence, according to Lemma \ref{pro-centroid}, there exists a subset $O'\subseteq T$ that include at most $k$ point satisfying $\mathtt{Cost}(P, O', \phi^*_O) \le (1+\epsilon)OPT$. Therefore, we have $\mathtt{Cost}(P, T, \phi_T^*) \le \mathtt{Cost}(P, O', \phi_{O'}^*) \le \mathtt{Cost}(P, O', \phi_{S_{opt}}^*) \le (1+\epsilon) OPT $.
\end{proof}
% \begin{observation}
%     Every $\tilde{s} \in S_{opt}$ is the centroid of $N_O(V)$, \emph{i.e.}, 
% %$o = \frac{\sum_{p\in N_O(V)}p\cdot \phi_{S_{opt}}(p,\tilde{s})}{\sum_{p\in N_O(V)}}\phi_{S_{opt}}(p)$    
%     { $o = \frac{\sum_{p\in N_O(V)}p\cdot \phi_{S_{opt}}(p,\tilde{s})}{\sum_{i=1}^{m} w_O^{(i)}(V)}$}.
% \end{observation}
% \begin{lemma}
%     There exists a subset $\tilde{T} \subseteq T'$ such that $\mathtt{supp}(T) \le k$ and $\mathtt{Cost}(P, \tilde{T}, \phi_T^*) \le (1+\epsilon) OPT$.
% \end{lemma}
% Combining Lemma~\ref{the-ratio} and Lemma~\ref{lem-app-ratio-centroid-set}, we can immediately arrive at the Theorem~\ref{approx-ratio}.
% \vspace{-0.1in}
Through combining Lemma~\ref{the-ratio} and Lemma~\ref{lem-app-ratio-centroid-set}, we can immediately obtain Lemma~\ref{cor-approx}.
% {\color{black}regarding $\rho$ as a bounded constant number.}
% Combining Lemma~\ref{the-ratio} and Lemma~\ref{lem-app-ratio-centroid-set}, we immediately obtain Lemma~\ref{cor-approx} regarding $\rho$ as a constant.
% achieve a ($1+4\rho+\epsilon$)-approximate solution of the fractional $(\alpha,\beta)$ fair $k$-means.

% \begin{corollary}
%     Equipped with the PTAS of vanilla $k$-means algorithm by \citet{cohen2019local} and $\epsilon$-approximate centroid set by \citet{matouvsek2000approximate} in low dimensional Euclidean space, the cost of the solution returned by Algorithm~\ref{alg-fair} is at most $(5 + O(\epsilon))OPT$. And the running time is $poly(n, k)$.
% \end{corollary}
{
\begin{lemma}
    Equipped with the $\epsilon$-approximate centroid set by \citet{matouvsek2000approximate}, the cost of the solution returned by Algorithm~\ref{alg-fair} is at most $(1+4\rho + O(\epsilon))OPT$. Furthermore, by utilizing the PTAS of vanilla $k$-means algorithm, the cost of the solution is at most $(5 + O(\epsilon))OPT$.
    \label{cor-approx}
\end{lemma}
}
%\begin{remark}
 %   In practice, the cost of assigning $P$ to $\epsilon$-approximate centroid set $T$ is much smaller than $(1+\epsilon)OPT$. In other words, the factor $\eta$ could less than $1+\epsilon$, which yields nearly $2\rho$ approximate factor if $\eta$ is small.
%\end{remark}

% {\color{green}
% \begin{remark}
%     The time complexity can be further reduced by using $\epsilon$-coreset for $(\alpha, \beta)$-fair $k$-means technique, \emph{e.g.}, the size of the coreset of \citep{braverman2022power} is no more than $\tilde{O}(\frac{k^3}{\epsilon^6})$. If the running time of constructing a coreset is $\mathcal{T}_{core}$ (near linear in $n$), {the overall time complexity can be reduced to $\mathcal{T}_{core} + \tilde{O}(\frac{k^3}{\epsilon^6}\log \frac{k^3}{\epsilon^6} + \frac{k^3}{\epsilon^{6+d}}\log (1/\epsilon)) + \mathcal{T}_{LP} + \tilde{O}(\frac{k^4}{\epsilon^6}) + \mathcal{T}_{means}$.} 
%     %Besides, since the the number of variables of linear programming is less, $\mathcal{T}_{LP}$ can also be reduced.
% \end{remark}
% }

\textbf{Rounding for integral solution.} Note that Lemma~\ref{cor-approx} only guarantees a fractional solution. Recall the ``violation factor'' introduced in Section~\ref{sec-pre}. According to the rounding method  proposed in \citep{bera2019fair}, a fractional solution of Problem~\ref{pro-fair-means} can be rounded to be integral with $(3\Delta + 4)$ violation, where $\Delta$ is the maximum number of groups a point can join in (e.g., if a point can belong to three groups, $\Delta$ should be equal to $3$). Their main idea is to reduce the fair assignment problem to the \emph{minimum degree-bounded matroid basis} (MBDMB) problem, and then solve the MBDMB by iteratively solving a linear program (LP). %{(e.g., if the groups are mutually disjoint, $\Delta$ should be $1$)}. 
In the current article, we further propose a new rounding method that can improve this violation factor to ``$2$'' when assuming $\Delta = 1$, \emph{i.e.}, the groups are mutually \textbf{disjoint}, and the each point belongs to exactly one group (if using the method of \citep{bera2019fair}, the factor should be $7$). 
Actually, it is natural to assume that the groups are disjoint, e.g., each person may belong to one race. 
Fair clustering problem in disjoint groups has also been  studied in  
\citet{bercea2018cost, wu2022new, chierichetti2017fair}. 
Our key idea is building a \textbf{``hub-guided'' minimum cost circulation} problem. Roughly speaking, we utilize a set of carefully designed ``hubs'' in a transportation network, for guiding the integral fair matching  between the input points and the obtained cluster centers. We show the result in Lemma~\ref{lem-rounding}, and place the proof to Appendix~\ref{rounding-alg} due to the space limit. 
%Under this assumption, we propose a new rounding algorithm, which achieves a $2$-violation rounding for any fractional $(\alpha,\beta)$-fair $k$-means solution and 
%while the cost of the solution does not increase. We reduce the  $(\alpha,\beta)$-fair $k$-means to a Minimum Cost Flow Problem (MCFP) and further gather an optimal integral solution based on a fractional solution. This algorithm is inspired by a technique in \citep{Ding2015soda} for solving the $l$-diversity $k$-means problem, while having some fundamental differences with their method. Due to the space limit, we place Lemma~\ref{lem-rounding} here and leave the details in our appendix.

\begin{lemma}
\label{lem-rounding}
    If the groups are mutually disjoint, one can round the fractional solution returned by Algorithm~\ref{alg-fair} to be integral with at most $2$-violation, while the cost does not increase.
\end{lemma}

Finally, Theorem~\ref{the-approx} can be obtained by combining either the rounding method from \citep{bera2019fair} for general case, or Lemma~\ref{lem-rounding} for disjoint case. 

%combining Lemma~\ref{the-ratio}, Lemma~\ref{cor-approx}, and Lemma~\ref{lem-rounding}.

\textbf{Overall time complexity.} As we mentioned in Remark~\ref{rem-complexity-matouvsek}, computing an $\epsilon$-approximate centroid set of $P$ needs $O(n\log n + n\epsilon^{-d}\log (1/\epsilon))$ time. 
The adjustment of the location of $T$ can be completed in $O(kn)$ time.
Suppose the time complexities of linear programming, vanilla $k$-means are denoted by $\mathcal{T}_{LP}$ and $\mathcal{T}_{means}$, respectively.  The overall time complexity of Algorithm~\ref{alg-fair} is  $O(n\log n + n\epsilon^{-d}\log (1/\epsilon)) + \mathcal{T}_{LP} + O(kn) + \mathcal{T}_{means}$. It is worth noting the the complexity can be further reduced by using the assignment preserving coreset ideas \citep{huang2019coresets,braverman2022power,DBLP:journals/jcss/BandyapadhyayFS24}. By doing this, we need to introduce an extra running time for coreset construction, which is linear to $n$, but we can compress the data size from $n$ to $poly(k, \epsilon)$.

\subsection{Extension to \texorpdfstring{$k$}{}-sparse Wasserstein Barycenter}
\label{sec-kwb}
% \vspace{-0.1in}
A cute property of Algorithm~\ref{alg-fair} is that it can be easily extended to address the $k$-sparse WB problem. 
Recall the definition of $k$-sparse WB in Problem~\ref{prob-wb}.
% First, we give the definition of Wasserstein distance and $k$-sparse Wasserstein Barycenter.
The given $m$ distributions can be viewed as $m$ groups of weighted points. And the sum of Wasserstein distances between barycenter and given distributions can be rewritten as the sum of squared Euclidean distances from $P$ to the centers. Moreover, the flows induced by Wasserstein distances between barycenter and the given distributions can implicitly ensure the fairness, \emph{i.e.}, for each point $s$ in barycenter,  $w^{(i)}(s) = \frac{1}{m}w(s)$ for any $i\in [m]$. Namely,
we can  directly perform our ``Relax and Merge'' framework by setting $\alpha_i = \beta_i = 1/m$. 
% We use Wasserstein Barycenter (without sparsity constraint), which can be computed in polynomial time in fixed dimensional Euclidean space~\citep{altschuler2021wasserstein}, as the relaxed solution of $k$-sparse WB.
First, we calculate the $\epsilon$-approximate centroid (here we ignore the weight of points) set to obtain $T$, then we use $T$ as the support of the Barycenter to run a ``fixed support''
% Barycenter~\citep{claici2018stochastic, cuturi2014fast, cuturi2016smoothed, ge2019interior, kroshnin2019complexity, lin2020fixed} algorithm to obtain the weight of $T$. 
WB algorithm~\citep{claici2018stochastic, cuturi2014fast, cuturi2016smoothed, lin2020fixed} to obtain the weights of $T$ (due to the space limit, we leave some details on fixed support WB algorithms to Appendix~\ref{fix-wb}). 
Finally, we run a vanilla $k$-means algorithm on $T$ to obtain the $k$-sparse solution. 
% Similar to Lemma~\ref{lem-app-ratio-centroid-set} , we can prove that the approxiamte centroid set also provide a ``good enough'' candidate set for $k$-sparse WB. 
% \begin{lemma}
%     If $T$ is an $\epsilon$-approximate centroid set of $\mathtt{supp}(\cup_{i=1}^m P^{(i)})$, there exists $T' = \{t_t\in T | i=1,\cdots ,k\}$ such that the cost of fix support WB on $T'$ is at most $(1+\epsilon)$ times of optimal $k$-sparse WB.
%     \label{lem-wb-centroid}
% \end{lemma}
% The proof of Lemma~\ref{lem-app-ratio-centroid-set} is not trivial and we leave the proofs in the complementary materials owing to the space limit. 
% Finally, we have the following theorem and we leave the proof in the complementary.
\begin{theorem}
    If $T$ is 
    % Wasserstein Barycenter
    an $\epsilon$-approximate centroid set
    of $\cup_{i=1}^m P^{(i)}$, Algorithm~\ref{alg-fair} returns a {$(1+4\rho + O(\epsilon))$-approximate} solution for $k$-sparse Wasserstein Barycenter problem.
    \label{the-wb}
\end{theorem}
% \vspace{-0.15in}
\subsection{Strictly Fair \texorpdfstring{$k$}{}-means without Violation }
\label{sec-strict}
% \vspace{-0.1in}
% {change the subtitle, add transition}

Since the strictly fair $k$-means is a special case of $(\alpha,\beta)$-fair $k$-means, by using Algorithm~\ref{alg-fair} and the rounding technique introduced by Section~\ref{sub-frac}, we can obtain an integral solution but with certain violation. In this section, we consider how to obtain an integral solution with no violation. Specifically, we compute the fairlet decomposition~\citep{chierichetti2017fair} for the input groups and use its centroids as the relaxed solution $T$ rather than $\epsilon$-approxiamte centroid set. First, we give the definition of fairlet decomposition for multiple groups, which extends the original definition of~\citep{chierichetti2017fair} from two groups to multiple groups. 

\begin{definition}[Fairlet Decomposition]
 Given a dataset $P$ that has $m$ equal-sized disjoint groups, We say a set $G$ of $m$ points is a fairlet of $P$, if $G$ contains exactly one point from each group of $P$. A set $\mathcal{G}$ of $n/m$ fairlets is a fairlet decomposition of $P$, if all fairlets in $\mathcal{G}$ are disjoint, where $n/m$ is the number of points in each group of $P$.
    % A fairlet $f$ consists of $m$ points from different groups. A fairlet decomposition $F$ of $P$ is a group of $n$ disjoint fairlets where every points in these fairlets belong to $P$.
\end{definition}

We define the cost of fairlet decomposition $\mathcal{G}$ as $\mathtt{Cost_{fairlet}}(\mathcal{G}) = \sum_{G\in \mathcal{G}} \sum_{p\in G} \Vert p- \mathtt{Cen}(G) \Vert^2$. It is easy to know that fairlet decomposition is indeed a solution of strictly fair $n/m$-means. Hence, we can still use the ``Relax and Merge'' technique: regard the centroids of fairlets in fairlet decomposition as a relaxed solution, and then run $\rho$-approximate vanilla $k$-means algorithm on these centroids. So, we reduce the strictly fair $k$-means problem to the fairlet decomposition problem.  We propose Algorithm~\ref{alg-balance}, which first computes a $2$-approximate fairlet decomposition and then generates a $(2 + 6\rho)$-approximate integral solution for strictly fair $k$-means.
% The high level idea of our algorithm is that compute a fairlet decomposition first, then run vanilla $k$-means algorithm on the centroids of fairlets. The final assignment is yielded by the fairlet decomposition so that we do not need linear programming to obtain the assginment matrix. The details of the algorithm is shown by Algorithm~\ref{alg-balance}.
% And the cost of a fairlet decomposition of $P$ is the sum of the cost of each fairlet. Under this definition, we define the optimal fairlet decomposition $FD^*$ is the fairlet decomposition who has the lowest cost, denoted by $OPT_{FD}$.

\begin{algorithm}
  \SetAlgoLined
  \caption{\sc{Stricytly Fair $k$-means}}
  \label{alg-balance}
  \KwIn{The dataset $P = \cup_{i=1}^m P^{(i)}$, $k$}
  %\KwOut{E}

  % Initialize the best fairlet decomposition $F=\emptyset$.
  
  % Initialize the best cost $v = \infty$.

  % Initialize the pivital group number $v$.

  \For{$i=1$ to $m$}
  {
    \For{$j=1$ to $m$ and $i\ne j$}
    {
        Compute the perfect one-to-one matching $\tau_{ij}$ between $P^{(i)}$ and $P^{(j)}$ {by using the Hungarian algorithm~\citep{kuhn1955hungarian}}. For each point $p\in P^{(i)}$, the point matched with $p$ {in $P^{(j)}$} is denoted as $\tau_{ij}(p)$.
        % that minimize $\mathcal{W}^2(P^{(i)}, P^{(j)})$.
    }
    Construct a fairlet decomposition $\mathcal{G}_{i}$ {(initially empty)} according to the matchings: for each point $p\in P^{(i)}$, add the fairlet $\{\tau_{i1}(p), \tau_{i2}(p), \cdots, \tau_{im}(p)\}$ to $\mathcal{G}_{i}$.

    % \If{the cost of $F_{new}$ is smaller than $F$}{
    %     $F \leftarrow F_{new}$.
    % }
  }

    Choose $\mathcal{G}_v$ where $v = \arg \min_i \sum_{p\in P^{(i)}} \sum_{j=1}^m \Vert p - \tau_{ij}(p) \Vert^2$ as $\mathcal{G}$.
  
    Construct the relaxed solution $T = \{ \mathtt{Cen}(G) \mid G \text{ is any fairlet of } \mathcal{G} \}$.
  
  Run a $\rho$-approximate $k$-means algorithm on $T$, and obtain the solution $S$.

  {\textbf{Integral assignment:}} Assign all the points according to the fairlet decomposition $\mathcal{G}$, \emph{i.e.}, if a point $p$ belongs to some fairlet $G$, then assign $p$ to $\mathcal{N}(\mathtt{Cen}(G), S)$. 
  
  \Return{S and the obtained integral assignment}
\end{algorithm}
% \vspace{-0.2in}
\begin{theorem}
    Algorithm \ref{alg-balance} returns a $(2 + 6\rho)$-approximate integral solution of strictly fair $k$-means.
    \label{the-ratio-strict}
\end{theorem}
% % \vspace{-0.1in}
% {\color{black}The proof of Theorem~\ref{the-ratio-strict} can be directly obtained by Lemma~\ref{the-ratio} and the following lemma when we directly substitute $\eta = 2$ into Lemma~\ref{the-ratio}.}
To prove Theorem~\ref{the-ratio-strict}, we need to prove the following lemma, which shows that $\mathcal{G}$ is a 2-approximate fairlet decomposition. Then, we can use the same idea of Lemma~\ref{the-ratio} to obtain Theorem~\ref{the-ratio-strict}. Recall that Lemma~\ref{the-ratio} shows that if we have a relaxed solution $T$ with a bounded cost $\eta \cdot OPT$, then the merged solution will have constant approximate ratio. Here, $T$ obtained by Algorithm~\ref{alg-balance} also provides a relaxed solution whose cost does not exceed $2OPT$. Hence, after we merge $T$ and obtain $S$, the approximate ratio should no more than $(\eta + (2\eta + 2)\rho) = 2+6\rho$. Furthermore, if we use PTAS for $k$-means, the overall approximate ratio of Algorithm~\ref{alg-balance} is $8+O(\epsilon)$.

% . The integral assignment returned by Algorithm~\ref{alg-balance} can be seen as an assignment matrix $\phi_T$ and Lemma~\ref{lem-ratio-fd}} impplies that $\mathtt{Cost}(P, T, \phi_T) \le 2OPT$.
\begin{lemma}
    If $\mathcal{G}$ is the fairlet decomposition obtained by Algorithm~\ref{alg-balance}, then $\mathtt{Cost}_{fairlet}(\mathcal{G})\le 2 OPT$.
    \label{lem-ratio-fd}
\end{lemma}
% \vspace{-2in}
\begin{proof}
    Suppose $G$ is a fairlet, and we use $G^{(i)}$ to denote the point in $G$ and belongs to group $P^{(i)}$, \emph{i.e.}, $G^{(i)} = G\cap P^{(i)}$. We use $\mathcal{G}_{OPT}$ to  denote {the optimal fairlet decomposition that has the lowest cost (we cannot obtain $\mathcal{G}_{OPT}$ in reality, and here  we just use it for conducting our analysis)}. For each $p\in P$, let $\mathcal{G}_{OPT}(p)$ denote the fairlet of $\mathcal{G}_{OPT}$ that $p$ belongs to, \emph{i.e.}, $p \in \mathcal{G}_{OPT}$. Suppose that $P^{(u)}$ is the ``closest'' group to $\mathcal{G}_{OPT}$, \emph{i.e.} 
 $u = \arg \min_{i\in [m]} \sum_{p\in P^{(i)}} \Vert \mathtt{Cen}(\mathcal{G}_{OPT}(p)) - p \Vert^2$.
    % $\mathcal{W}^2(P^{(v)}, F^*)$ 
 % (recall that $\mathcal{W}$ is the Wasserstein Distance) 
 % is the smallest among all groups. 
 % For each fairlet $f\in F^*$, let $\sigma^{(i)}(f)$ denote the point in $P^{(i)}$ belongs to the optimal fairlet decomposition $F^*$. 
 We have
 \begin{eqnarray}
 \begin{aligned}
     &\mathtt{Cost}_{fairlet}(\mathcal{G}) = \sum_{G\in \mathcal{G}} \sum_{p\in G} \Vert p- \mathtt{Cen}(G) \Vert^2 \\
     & \le \sum_{G\in \mathcal{G}} \sum_{p\in G} \Vert p- \mathtt{Cen}(G) \Vert^2 + m \sum_{G\in \mathcal{G}} \Vert G^{(v)}- \mathtt{Cen}(G) \Vert^2   
 \end{aligned}
\label{eq-fd}
 \end{eqnarray}
 According to Proposition~\ref{lem-centroid}, {the right side of (\ref{eq-fd}) equals to $\sum_{p\in P^{(v)}} \sum_{j=1}^m \Vert p - \tau_{vj}(p) \Vert^2$}, so we have
 % $\mathtt{Cost}(F) \le \sum_{p\in P^{(v)}} \sum_{j=1}^m \Vert p - \tau_{vj}(p) \Vert^2
 %    \le \sum_{p\in P^{(u)}} \sum_{j=1}^m \Vert p - \tau_{vj}(p) \Vert^2
 %    \le \sum_{p\in P^{(u)}} \sum_{j=1}^m \Vert p - (F^*(p))^{(i)} \Vert^2$.
 % $\mathtt{Cost}(F) \le \sum_{p\in P^{(v)}} \sum_{j=1}^m \Vert p - \tau_{vj}(p) \Vert^2$
% \vspace{-0.4in}
    $\mathtt{Cost}_{fairlet}(\mathcal{G}) \le $
    \begin{equation}
        \sum_{p\in P^{(v)}} \sum_{j=1}^m \Vert p - \tau_{vj}(p) \Vert^2 \le \sum_{p\in P^{(u)}} \sum_{j=1}^m \Vert p - \tau_{uj}(p) \Vert^2 
    \le \sum_{p\in P^{(u)}} \sum_{j=1}^m \Vert p - (\mathcal{G}_{OPT}(p))^{(j)} \Vert^2.
    \end{equation}
    
 % $ F \le \sum_{p\in P^{(u)}} \sum_{i=1}^m \Vert p-\sigma^{(i)}(F^*(p)) \Vert ^2 $. Since $\mathtt{cen}(F^*(p))$ is the centroid of $F^*(p)$, 
 {\color{black}The first inequality holds because $v = \arg \min_i \sum_{p\in P^{(i)}} \sum_{j=1}^m \Vert p - \tau_{ij}(p) \Vert^2$. And the last inequality holds because $\tau$ is the perfect one-to-one matching.} Using Proposition~\ref{lem-centroid} again, we have $\mathtt{Cost}_{OPT}(\mathcal{G}) \le$
    \begin{equation}
    \begin{aligned}
         \sum_{p\in P^{(u)}} \sum_{j=1}^m \Vert \mathtt{Cen}(\mathcal{G}_{OPT}(p))-(\mathcal{G}_{OPT}(p))^{(j)} \Vert ^2+m\sum_{p\in P^{(u)}} \Vert \mathtt{Cen}(\mathcal{G}_{OPT}(p)) - p \Vert^2.\label{for-lem4-1}
    \end{aligned}
    \end{equation}
    Note that $\mathcal{G}$ is the optimal fairlet decomposition, as well as the optimal strictly fair $n/m$-means solution, so the first term of (\ref{for-lem4-1}) should be at most $OPT$. As for the second term, since $P^{(u)}$ is the ``closest'' group to $\mathcal{G}$, it should be no larger than $m \cdot \frac{1}{m} OPT \le OPT$ (because the minimum distance ``$\sum_{p\in P^{(u)}} \Vert \mathtt{Cen}(\mathcal{G}_{OPT}(p)) - p \Vert^2$'' should not  exceed the average distance $ \frac{1}{m} OPT$). Overall, we complete the proof of Lemma~\ref{lem-ratio-fd}.
    % \begin{equation}
    %     \begin{aligned}
    %         Cost &\le \sum_{p\in P^{(j)}} \sum_{i=1}^m \Vert p-\sigma^{(i)}(F^*(p)) \Vert ^2 = \sum_{p\in P^{(j)}} \Big[ \sum_{i=1}^m \Vert c(F^*(p))-\sigma^{(i)}(F^*(p)) \Vert ^2 + m\Vert c(F^*(p)) - p \Vert^2 \Big] \\
    %         &= \sum_{p\in P^{(j)}} \sum_{i=1}^m \Vert f-\sigma^{(i)}(p) \Vert ^2 +\sum_{p\in P^{(j)}} m\Vert f - p \Vert^2 \\
    %         &= OPT_{fd} + m \sum_{p\in P^{(j)}}\Vert f - p \Vert^2 \\
    %         &\le OPT_{fd} + m \cdot \frac{1}{m} OPT_{fd}\\
    %         &\le 2OPT_{fd}
    %     \end{aligned}
    % \end{equation}
\end{proof}

\vspace{-0.2in}
\section{Experiments}
\label{sec-exp}
% \vspace{-0.1in}
% \subsection{Experimental environment and datasets}
In this section, we perform the empirical evaluation on our algorithms.
Our experiments are conducted on a server equipped with Intel(R) Xeon(R) Gold 6154 CPU @ 3.00GHz CPU and 512GB memory. We implement our algorithms in C++ and python (with linear programming solver gurobi~\citep{gurobi}).
We use the following datasets which are commonly used in previous works: \textbf{Bank}~\citep{moro2014data}(4522 points with 5 groups), \textbf{Adult}~\citep{misc_adult_2} (32561 points with 7 groups), 
\textbf{Census}~\citep{zhou2002hybrid}(50000 points with 10 groups), \textbf{creditcard}~\citep{yeh2009comparisons} (30000 points with 8 groups), 
\textbf{Biodeg}~\citep{mansouri2013quantitative} (1055 points with 2 groups),
\textbf{Breastcancer}~\citep{misc_breast_cancer_wisconsin_(diagnostic)_17} (570 points with 2 groups), \textbf{Moons}~\citep{moons} (200 points with 2 groups), 
\textbf{Hypercube}(200 points with 2 groups), 
\textbf{Cluto}~\citep{karypis1999chameleon} (800 points with 8 groups), and \textbf{Complex} (800 points with 8 groups).  The last four datasets consist of disjoint and equal sized groups, so we can perform strictly fair $k$-means algorithms on them. We place the detailed information of these datasets in Appendix~\ref{sup-exp}. Regarding the selection of $\alpha$ and $\beta$, we set $\alpha_i = \beta_i = \frac{|P^{(i)}|}{|P|}$ and we also discuss more choices for $\alpha$ and $\beta$,  and provide more experimental results, including the part of $k$-sparse Wasserstein Barycenter, in the Section~\ref{sup-exp} of the appendix. We use $k$-means++~\citep{ostrovsky2013effectiveness} as the $k$-means solver in our Algorithm~\ref{alg-fair}.

% \begin{table}
% \caption{Datasets}
% \begin{tabular}{llrl} 
% Dataset & Coordinates & \begin{tabular}{r} 
% Sensitive \\
% attributes
% \end{tabular} & Protected groups \\ \toprule bank & age, balance, duration & marital & married, single, divorced \\
% \cline { 3 - 4 } & & default & yes, no \\
% \hline \multirow{2}{*}{ census } & age, education-num, & sex & female, male \\
% \cline { 3 - 4 } & final-weight, capital-gain, & race & Amer-ind, asian-pac-isl, \\
% & hours-per-week & & black, other, white \\
% \hline creditcard & age, bill-amt 1-6, & marriage & married, single, other, null \\
% \cline { 3 - 4 } & limit-bal, pay-amt 1 - 6 & education & 7 groups \\
% \hline census1990 & dAncstry1, dAncstry2, iAvail, & dAge & 8 groups \\
% \cline { 3 - 4 } & iCitizen, iClass, dDepart, iFertil, & iSex & female, male \\
% & iDisabl1, iDisabl2, iEnglish, & & \\
% & iFeb55, dHispanic, dHour89 & & \\
% \bottomrule
% \end{tabular}
% \label{tab-dataset}
% \end{table}
% In our experiments, each test instance is repeated 5 times and we report the average result.

\textbf{Results on $(\alpha,\beta)$-Fair $k$-means.} We compared the cost of $(\alpha,\beta)$-fair $k$-means of our Algorithm~\ref{alg-fair} and baselines. We choose the algorithm proposed by \citet{bera2019fair} (denoted by NIPS19) and \citet{bohm2021algorithms} (denoted by ORL21) as the baselines. The construction of an $\epsilon$-approximate centroid set is a theoretical algorithm that can be replaced by some efficient methods in practice. In our experiments, we adopted the alternative implementation of \citet{kanungo2002local}, which combines the kd-tree~\citep{friedman1977algorithm} and a sampling technique. 
% In our experiment, we do not  explicitly construct an $\epsilon$-approximate centroid set as \citet{matouvsek2000approximate} since the size of the set is so large when the dimension is high that the LP procedure will be very time comsuming. Instead, we use the alternative implementation of \citet{kanungo2002local}, which combine the kd-tree~\citep{friedman1977algorithm} and a sampling technique. 
Figure~\ref{fig:fractional_cost} shows that our algorithm gives the lowest cost of $(\alpha, \beta)$-fair $k$-means, indicating that Algorithm~\ref{alg-fair} can find better center locations. {This improvement is possible due to that
our method considers the fairness information of groups when choosing the locations of centers. 

%before running vanilla $k$-means, which enables our method to find better solution.
}
% \vspace{-0.1in}
\begin{figure}[h]
    \centering
    \includegraphics[width=\textwidth]{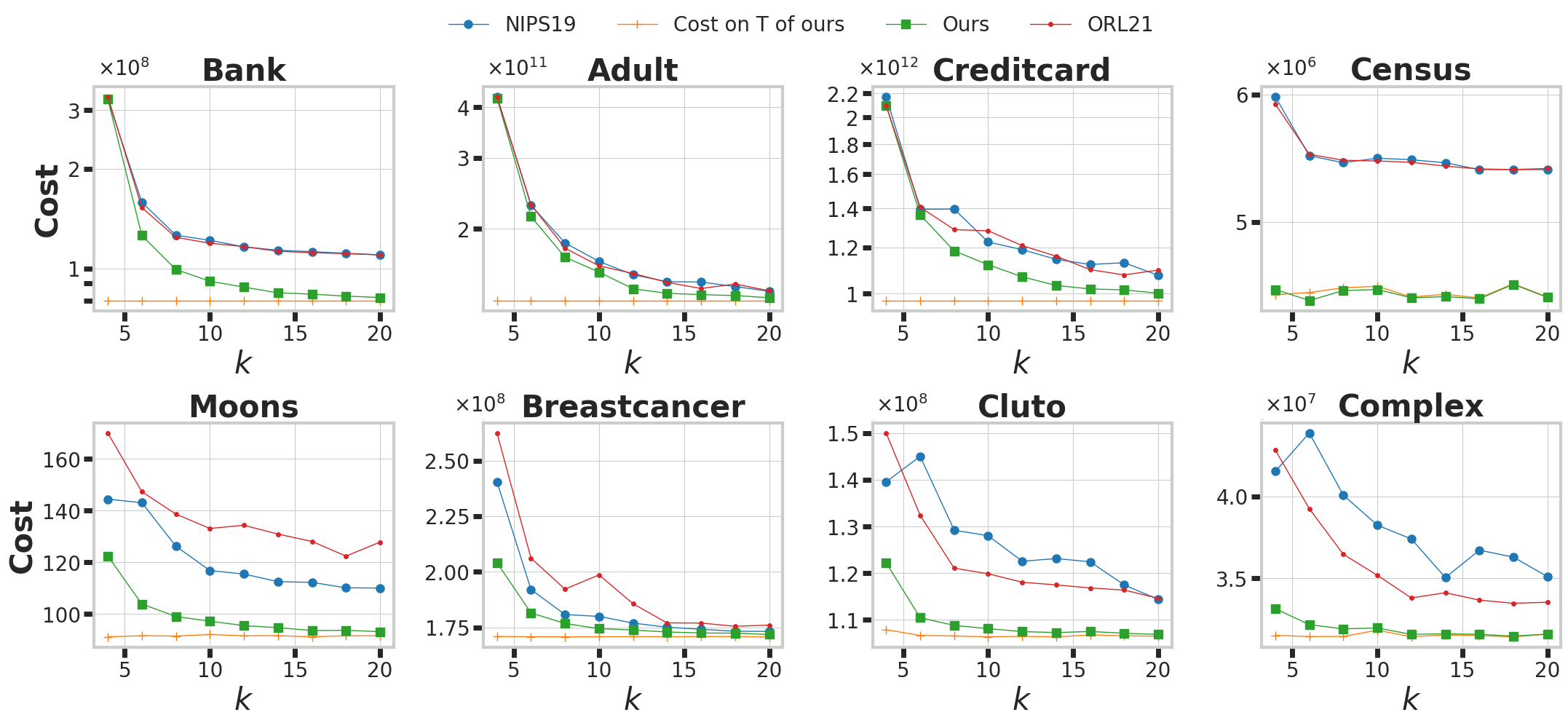}
    
    \caption{{The cost obtained by the algorithms with different $k$.}}
    \label{fig:fractional_cost}
\end{figure}

\textbf{Results on Strictly Fair $k$-means.}
% We compare our strictly fair $k$-means algorithm with ORL21~\citep{bohm2021algorithms} as a baseline, which is the best algorithm of strictly fair $k$-means for multiple groups to the best of our knowledge. 
We compare our strictly fair $k$-means algorithm with the state-of-the-art algorithm ORL21~\citep{bohm2021algorithms}. Both ORL21 and Algorithm~\ref{alg-balance} can return integral solution with no violation. Figure~\ref{fig:integral_cost} shows that our method has significant advantage in terms of the clustering cost.
\begin{figure}[h]
    \centering
    \includegraphics[width=\textwidth]{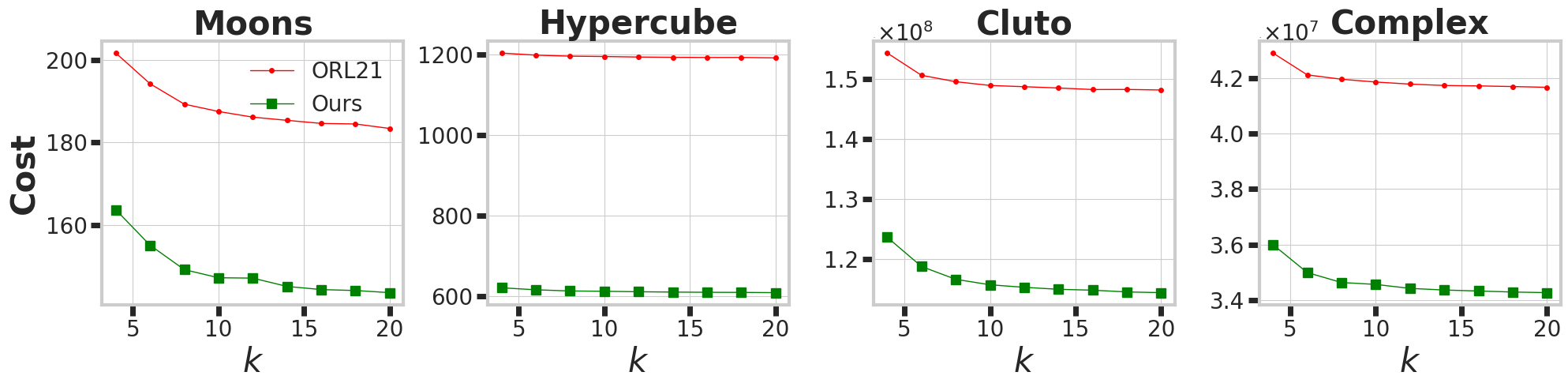}
    \caption{The cost of strictly fair $k$-means.}
    \label{fig:integral_cost}
\end{figure}
\section{Conclusion}
% \vspace{-0.1in}
% In this paper, we consider the $(\alpha, \beta)$-fair $k$-means problem. Our framework returns a fractional solution for general case and a integral solution for strictly fair case with better approximate solution. There are some important problems still open: how to obtain integral approximate solution of general case without violation? Is there exists PTAS in low dimensional space? We are aware of the previous PTAS of fair $k$-means~\citep{bohm2021algorithms, schmidt2020fair, ding2020unified}, but their methods have a exponential time complexity of $k$.
In this paper, we utilize the insight on the relationship between the fair $k$-means problem and a classic geometric structure, $\epsilon$-approximate centroid set, for developing a novel ``Relax and Merge" framework. It can achieve a $(1+4\rho+O(\epsilon))$ approximation ratio of fair $k$-means and $k$-sparse Wasserstein Barycenter problems, which improves the current state-of-the-art approximation guarantees. 
%We also conduct a set of experiments to illustrate the effectiveness of our proposed method in practice. 
There still exists some open problems: how to obtain an integral approximate solution of general case without violation? In addition, is it possible to extend our `Relax and Merge" framework to other types of clustering problems, such as  the  proportionally fair clustering~\citep{chen2019proportionally} and  socially fair $k$-means clustering~\citep{ghadiri2021socially}.

\section{Acknowledgement}
The authors want to thank Prof. Lingxiao Huang for the meaningful discussion about the rounding algorithm in Appendix~\ref{rounding-alg}.

%We are aware of the previous PTAS of fair $k$-means~\citep{bohm2021algorithms, schmidt2020fair}, but their methods have an exponential time complexity in $k$. So it is also interesting to consider designing  
%FPTAS in low dimensional space. 

%Empirical experiments demonstrate that our algorithms can achieve a lower clustering cost than the baselines.

\bibliography{main_paper}
\bibliographystyle{plainnat}

\newpage
\appendix

\section{\texorpdfstring{$\epsilon$}{}-approximate centroid set}
\label{sec-centroid-set}
The algorithm of constructing an $\epsilon$-approximate centroid set is proposed by \citet{matouvsek2000approximate}. Here we briefly introduce the idea. 
First, we use a quadtree to partition the space into hierarchical cubes. At each level of the tree, we construct a grid. The length of the grid is set to ensure that the grid points can always cover all approximate centroids of all cubes at this level.  The approximate centroid set is the union of all grid points across all levels.

In Figure~\ref{fig-centroid}, we visually illustrate the difference between the $k$-means clustering center and the fair $k$-means clustering center.
The vanilla $k$-means induces a Voronoi diagram, so that every $k$-means center is located at the centroid of a $k$-means cluster. However, a fair $k$-means center can be located at the centroid of any potential cluster that satisfies the fairness constraints. The $\epsilon$-approximate centroid set structure can help us to find these potential centroids and preserves the fairness constraints for the later procedures.

\begin{figure}[ht] %比如{r}{0.5\textwidth}
\centering
\includegraphics[width=0.8\textwidth]{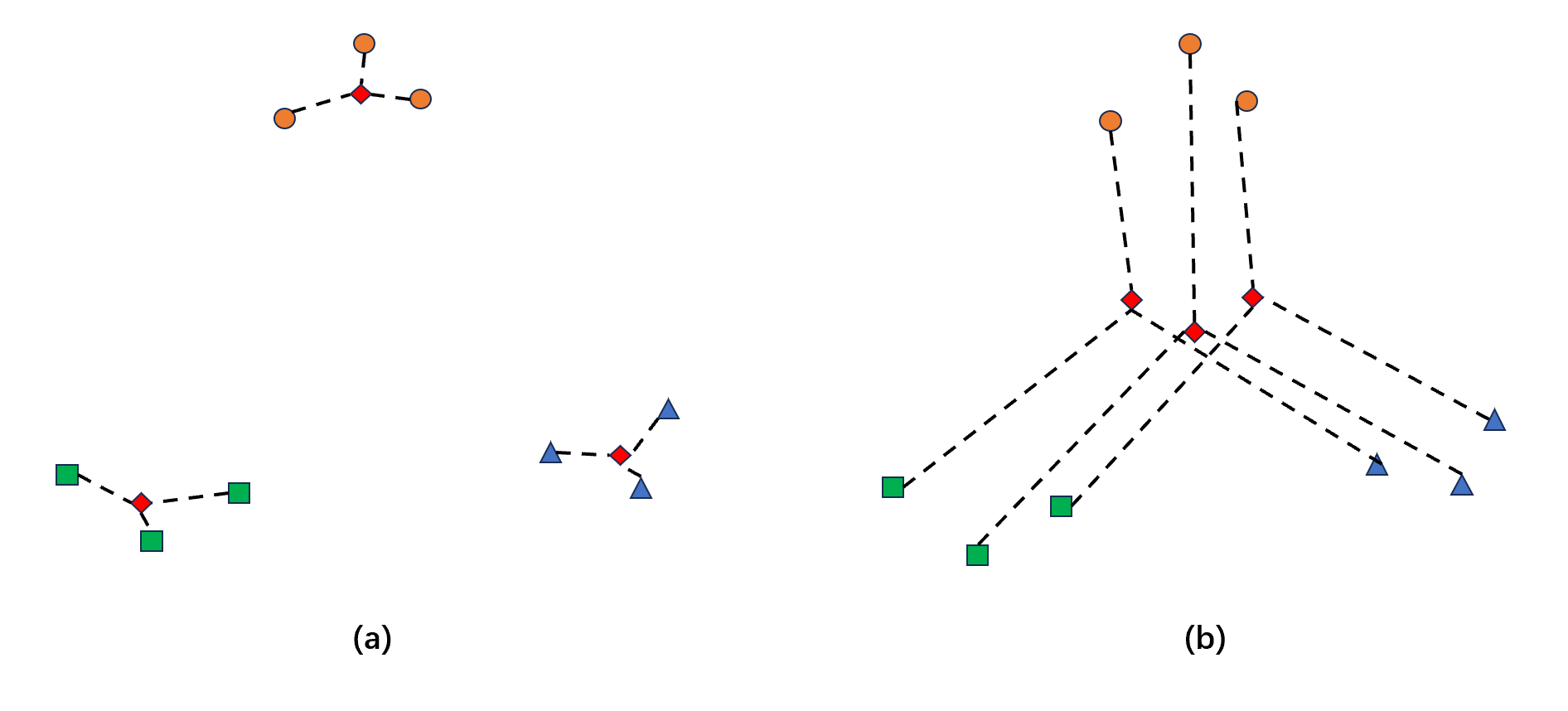}

\caption{The difference between the location of $k$-means clustering centers and the fair $k$-means clustering centers. The input dataset contains 3 different groups represented by orange, blue, and green points respectively. The red diamonds represent the cluster centers under different assumptions for the clustering problem. (a) shows the clustering result of $k$-means, while (b) shows the clustering result of fair $k$-means.}
\label{fig-centroid}
\end{figure}

\section{Omitted Proofs}
\label{app-omitted}

\paragraph{Theorem~\ref{the-wb}}If $T$ is 
    % Wasserstein Barycenter
    an $\epsilon$-approximate centroid set
    of $\cup_{i=1}^m P^{(i)}$, Algorithm~\ref{alg-fair} returns a $(1+4\rho + O(\epsilon))$-approximate solution for $k$-sparse Wasserstein Barycenter problem.

To prove this theorem, we need the following lemmas.
\begin{lemma}
    If $T$ is 
    an $\epsilon$-approximate centroid set
    of $\cup_{i=1}^m P^{(i)}$ and $w(t)$ for  each $t\in T$ is obtained by solving LP(\ref{eq-fair}), then $T$ is a $(1+O(\epsilon))$-approximate Wasserstein Barycenter.
    \label{lem-wb}
\end{lemma} 
\begin{proof}
    A critical fact is that there exist an optimal Wasserstein Barycenter $T^*$ such that all points of $T^*$ located in the centroid of some fairlet of $P$. This claim has been proved in \citep{anderes2016discrete} (Section 2, Equation 4). Therefore, if we calculate an $\epsilon$-approximate centroid set $T$, then $T$ can always cover the locations of $T^*$, \emph{i.e.}, $\mathtt{Cost}(P, T, \phi^*_T) \le (1+O(\epsilon))\mathtt{Cost}(P, T^*, \phi^*_{T^*}) \le (1+O(\epsilon))OPT$. So using the same proof idea with Lemma~\ref{lem-app-ratio-centroid-set}, we can obtain the conclusion of Lemma~\ref{lem-wb}.
\end{proof}
Combine Lemma~\ref{lem-wb} and Lemma~\ref{the-ratio}, we arrive at Theorem~\ref{the-wb}.
\section{The Rounding Technique}
\label{rounding-alg}

% \begin{lemma}
% \label{lem-rounding}
%     There exists an algorithm that can round a fractional solution of $(\alpha,\beta)$-fair $k$-means to integral with at most $2$-violation while the cost does not increase.
% \end{lemma}

% The sketch of our method is described as follows. Our rounding algorithm consists of three steps: constructing a network structure of Minimum Cost Flow Problem (MCFP), setting the parameter of the vertices/edges based on the fractional solution, and solving the MCFP above. Recall that the dataset $P$ consists of $m$ different  groups, \emph{i.e.}, $\boldsymbol{P=\cup_{i=1}^{m}P^{(i)}}$ and we assume that the groups are disjoint. By executing the Algorithm~\ref{alg-fair}, we obtain a center set $S$ and corresponding fractional assignment matrix $\phi^*_S$. Now we construct a network structure as Figure~\ref{fig-MCFP}. We create a "copy" of $S$, denoted by $S^{(i)}$, for each group $P^{(i)}$ while setting the cost of the arc from every $p \in P^{(i)}$ to an $s \in S^{(i)}$ to be $||p-s||^2$. The cost of the remaining arcs are $0$. The demand at the source is set to be $-n$, which means it is sending a flow of $n$ into the network. The demand at the sink is $n$ and the demands at all other vertices are $0$. Now, to form a Minimum Cost Flow instance, we only need to assign the capacity of each arc.

 Our rounding algorithm consists of three steps: constructing a network structure of Minimum Cost Circulation Problem (MCCP), setting the parameters of each edge based on a fractional solution obtained by Algorithm~\ref{alg-fair}, and solving the MCCP above. This reduction to MCCP is inspired by  \citet{Ding2015soda} (Section 4.3), while having some fundamental differences with their method. Our algorithm has different objectives compared to theirs, as it is based on a different approach to setting network parameters, and our method offers better time complexity guarantees. Our rounding algorithm requires only a single call to the minimum-cost circulation algorithm, and it can be completed in $O(n^3k^2)$ time even when using the vanilla Edmonds-Karp algorithm~\citep{edmonds2,edmonds1}. 
 
 The process of our algorithm is described as follows. Recall that the dataset $P$ consists of $m$ different  groups, \emph{i.e.}, $P=\cup_{i=1}^{m}P^{(i)}$ and we assume that the groups are disjoint. By executing the Algorithm~\ref{alg-fair}, we obtain a center set $S$ and corresponding fractional assignment matrix $\phi^*_S$. Now, in order to build a minimum cost circulation instance, we need to construct a network structure as Figure~\ref{fig-MCCP} and for each arc $(u,v)$, we should set the lower/upper bound of the flow $f(u,v)$ and its cost $c(u,v)$. We create a copy of $S$, denoted by $S^{(i)}$, for each group $P^{(i)}$. Each $S^{(i)}$ is a \textbf{"hub"} used for transit, specifically to receive weights from group $P^{(i)}$ and transmit them to $S$.   To facilitate understanding, we can imagine that each $s_l^{(i)} \in S^{(i)}$, where $l\in [k]$, 
and its corresponding $s_l \in S$ are in the same position, but only accepts the weights from group $P^{(i)}$.
We set $c(p_j^{(i)},s_l^{(i)})$, \emph{i.e.,} the cost of the arc from any $p_j^{(i)} \in P^{(i)}$, where $j\in [n^{(i)}]$, to an $s_l^{(i)} \in S^{(i)}$ to be $||p_j^{(i)}-s_l^{(i)}||^2$. The cost of the remaining arcs are $0$.

\begin{figure}[ht] %比如{r}{0.5\textwidth}
\centering
\includegraphics[width=\textwidth]{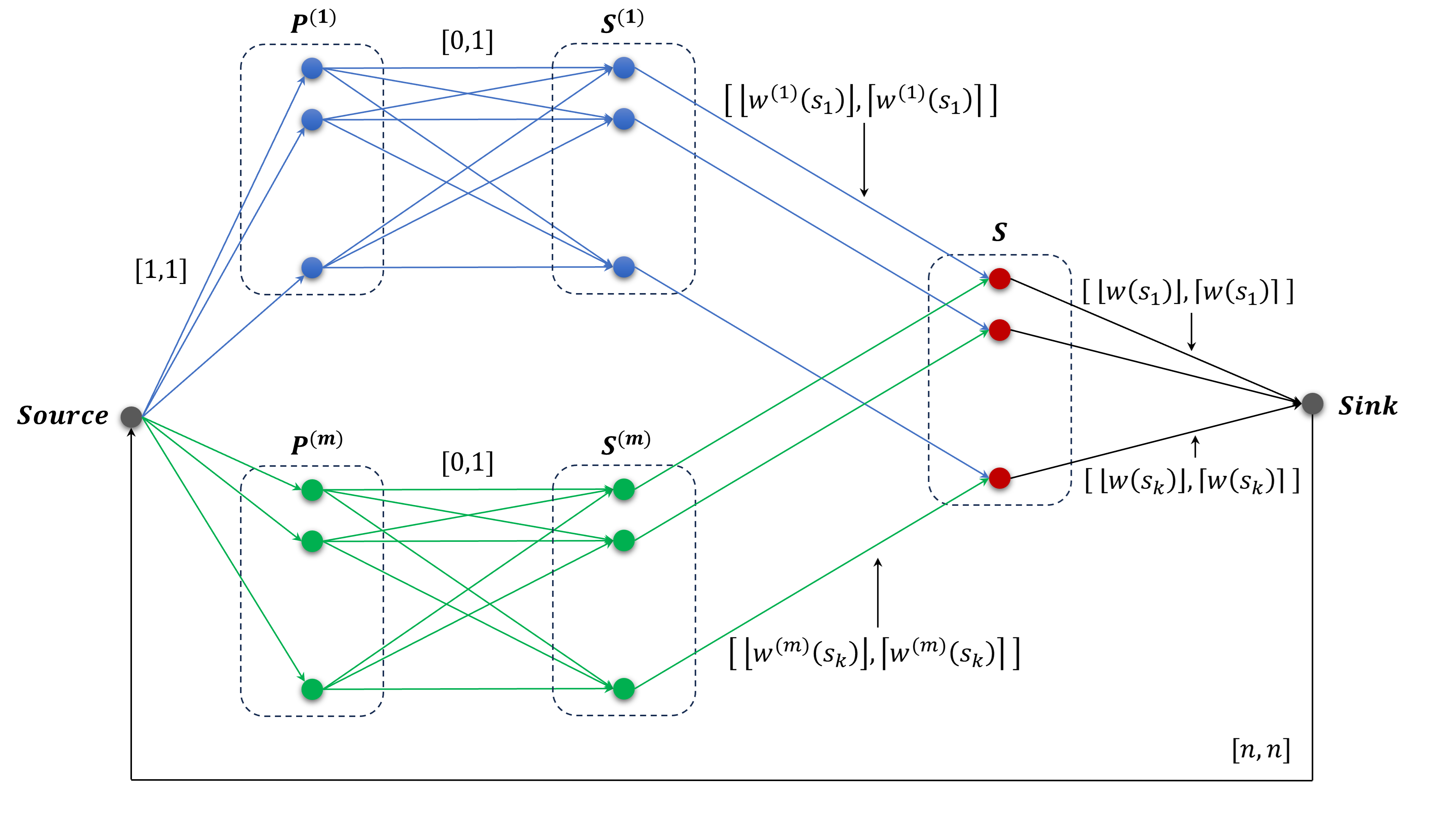}

\caption{The instance of the minimum cost circulation problem established through $(S, \phi_S^*)$. The upper and lower bounds of the flow for each arc are annotated in the graph.}
\label{fig-MCCP}
\end{figure}

Next, we set the lower bound and the upper bound of the flow on each arc, as shown in Figure~\ref{fig-MCCP}.
First, the flow from the "Source" node to each $p \in P$ is restricted to $1$, which means that each point $p \in P^{(1)}$ has a weight of $1$ to assign to $S^{(1)}$. Then, between each $P^{(i)}$ and its "hub" $S^{(i)}$, the flow from each $p_j^{(i)} \in P^{(i)}$ to each $s_l^{(i)} \in S^{(i)}$ is bounded by $[0,1]$. Here, the flow $f(p_j^{(i)}, s_l^{(i)})$ denotes the amount of the weight that assigned from $p_j^{(i)}$ to $s_l \in S$ in an $(\alpha,\beta)$-fair $k$-means solution. Subsequently, recall that in the solution $(S,\phi^*_S)$ we obtained before, the weight received by a center $s_l \in S$ from group $P^{(i)}$ is $w^{(i)}(s_l)$. We bound the flow $f(s_l^{(i)},s_l)$ by $\big{[}\lfloor w^{(i)}(s_l)\rfloor,\lceil  w^{(i)}(s_l)\rceil\big{]}$.
% which means the $s_l^{(i)}$ accepts the weights from group $P^{(i)}$ and bounded by the fairness constraints. 
Finally, the flow from each $s_l \in S$ to the "Sink" node is bounded by $\big{[}\lfloor w(s_l)\rfloor,\lceil  w(s_l)\rceil\big{]}$, and we set $f(Sink, Source)=n$ to form a circulation. At this point, we have established an instance of the minimum cost circulation problem, denoted by $MCCP(S,\phi^*_S)$. Obviously, we have the following observation:

\begin{observation}
    $\phi^*_S$ induces a feasible solution of $MCCP(S,\phi^*_S)$.
\end{observation}

The observation is straightforward because the flow induced by $\phi^*_S$ meet all the bounds applied to the flow. Then, we give the proof of Lemma~\ref{lem-rounding} mentioned in Section~\ref{sub-frac}.

\paragraph{Lemma~\ref{lem-rounding}.}
There exists an algorithm that can round a fractional solution of $(\alpha,\beta)$-fair $k$-means to integral with at most $2$-violation while the cost does not increase.

\begin{proof}
    It is known that the minimum cost circulation problem has an integrality property~\citep{cormen2009introduction}, which guarantees that if the arcs have integer capacities, there will always be an optimal solution with integer flow values on each arc. Utilizing an algorithm for minimum cost circulation problem or minimum cost flow problem (the two problems are equivalent), which  converges to an integer solution like Ford-Fulkerson~\citep{Ford_Fulkerson_1956}, we can obtain an integer optimal solution of $MCCP(S,\phi^*_S)$, which has a cost no larger than the solution induced by $\phi^*_S$. 
    
    Next, we prove that the assignment matrix, say $\phi'_S$, induced by the integer optimal solution of $MCCP(S,\phi^*_S)$ is a 2-violation assignment from $P$ to $S$. 
    Recall that we presented the definition of the violation factor in Section~\ref{sec-pre}: An assignment matrix $\phi_S$ is a $\lambda$-violation solution if $\beta_i \sum_{p\in P}\phi_S(p,s)  - \lambda \le \sum_{p\in P^{(i)}} \phi_S(p,s) \le \alpha_i \sum_{p\in P}\phi_S(p,s) + \lambda, \quad \forall s\in S, \forall i\in [m]$. According to the construction procedure of $MCCP(S,\phi^*_S)$, the lower bound of the flow $f(s_l^{(i)},s_l)$ is $\lfloor w^{(i)}(s_l)\rfloor$, which satisfies:

    \begin{equation}
        \begin{aligned}
    \lfloor w^{(i)}(s_l)\rfloor 
            &\geq \lfloor \alpha^{(i)} (\lceil w(s_l)\rceil -1)\rfloor \\
            &= \lfloor \alpha_i \lceil w(s_l)\rceil -\alpha_i\rfloor \\
            &\geq \lceil \alpha_i \lceil w(s_l)\rceil -\alpha_i \rceil -1 \\
            &\geq \big{(} \alpha_i \lceil w(s_l)\rceil -\alpha_i \big{)}-1. \\
        \end{aligned}
    \end{equation}

    Note that the upper bound of the flow $f(s_l,Sink)$ is $\lceil w(s_l)\rceil$ so we have:
        \begin{equation}
        \begin{aligned}
        \lfloor w^{(i)}(s_l)\rfloor 
        &\geq  \alpha_i \lceil w(s_l)\rceil -\alpha_i -1 \\
        &\geq \alpha_i \lceil w(s_l)\rceil -2, \\
        \end{aligned}
        \end{equation}
        and similarly,
        \begin{equation}
        \begin{aligned}
        \lceil w^{(i)}(s_l)\rceil 
        &\leq  \beta_i \lfloor w(s_l)\rfloor +\beta_i +1 \\
        &\leq \beta_i \lfloor w(s_l)\rfloor +2, \\
        \end{aligned}
        \end{equation}
        which indicates that $\phi'_S$ is a 2-violation assignment and complete the proof of Lemma~\ref{lem-rounding}.
    % \emph{i.e.}, $(S,\phi'_S)$ is 
\end{proof}

\section{Fixed Support Wasserstein Barycenter}
\label{fix-wb}
Given $m$ discrete distributions (weighted point sets, each set has total weight sum to $1$) $P^{(1)},\cdots, P^{(m)}$ and a $k$ locations set $T$ of WB, the objective of fixed support WB as follows:
\begin{eqnarray}
    \begin{aligned}
        \min_{x} \quad     & \frac{1}{m} \sum_{l = 1}^m \sum_{i = 1}^{n^{(i)}} \sum_{j=1}^{n^{(j)}} \Vert P^{(l)}_i - T_j \Vert ^2 x^{(l)}_{ij}                                      \\
      s.t. \quad & \sum_{j=1}^k x^{(l)}_{ij} = 1, \quad \forall l \in [m], \forall i \in [n^{(l)}] \\
      & \sum_{i=1}^{n^{(l)}} x^{(l)}_{ij} w(P^{(w)}_i) = y_j, \quad \forall l \in [m], \forall j \in [k] \\
      & \sum_{j=1}^k y_j = 1, \\
      & x^{(l)}_{ij} \ge 0, \quad \forall l \in [m], \forall i \in [n^{(l)}], \forall j \in k \\
      & y_j\ge 0 , \quad \forall j \in k
      \label{eq-wb}
    \end{aligned}
\end{eqnarray}
It is easy to see that {\color{black}fixed support WB} problem can be solved using linear programming method. Several existing works {\color{black}on solving LP (\ref{eq-wb}) including \citep{claici2018stochastic, cuturi2014fast, cuturi2016smoothed, lin2020fixed}.}

For the sake of completeness, we need to clarify how the solution to the $k$-sparse Wasserstein barycenter solution is guaranteed to be a distribution.
After we run Algorithm 1, we obtain the support $S$ (the locations of centers) of the returned solution and the assignment matrix $\phi_S^*$ (the transportation weight from $p = P^{(l)}_{i}$ to $f = S_j$ is denoted by $\phi^*_{S}(p,f) = x^{(l)}_{ij}$ in (\ref{eq-wb})). The key question is how to ensure that the summation of the weight of points in $S$ is equal to 1. Let us consider an arbitrary given distribution (or "group" in the context of fair $k$-means), e.g., $P^{(l)}$. For every facility $f$ in $S$, we define its weight $w(f) = \sum_{p\in P^{(l)}} \phi^*_{S}(p,f)$. This ensures that the total weight of $S$ must be equal to the total weight of $P^{(l)}$, which is 1 because $P^{(l)}$ is a distribution. The choice of $P^{(l)}$ can be arbitrary because, recall that $k$-sparse WB can be seen as a special fractional version of strictly fair $k$-means, meaning no matter which given distribution you choose, you will obtain the same weight distribution of $S$. The optimization will not change by setting the weight of $S$ because the weight of $S$ does not affect the cost. 

\section{Extend Algorithm~\ref{alg-fair} to \texorpdfstring{$k$}{}-Median and \texorpdfstring{$k$}{}-Means in General Metric Space}

Although we mainly consider the fair $k$-means problem in Euclidean space in this paper, for the sake of completeness, in this section, we illustrate how to extend our framework to solve $k$-median and $k$-means in general metric space. In summary, if the potential facility set is given, our framework achieves a $(1+2\rho)$-approximate solution for $k$-median ($(2+8\rho)$-approximate solution for $k$-means) in metric space, where $\rho$ is the approximation ratio for vanilla $k$-median ($k$-means) with a constant violation factor. If the metric space has a fixed doubling dimension~\citep{gupta2003bounded}, then equipped with existing PTAS for metric $k$-median and $k$-means~\citep{cohen2021near, cohen2019local, friggstad2019local}, the best approximation ratios our framework can achieve are $(3+O(\epsilon))$ for fair $k$-median and $(10+O(\epsilon))$ for fair $k$-means.

Unfortunately, our theoretical guarantees in general metric space are weaker than those of \citet{bera2019fair}, in which they obtained a $(\rho+2)$-approximation for $k$-median and a $(\sqrt{\rho} + 2)^2$-approximation for $k$-means. The obstacle to achieving a better approximation ratio for our framework is the "candidate set". In Euclidean space, we have an approximate centroid set. However, in general metric space, how can we obtain a candidate set that has similar properties to Proposition~\ref{lem-centroid}, which provides a more powerful tool than the basic triangle inequality? This is not only a potential future work of our framework but also an important open theoretical problem.

% When we consider $k$-clustering problem in general metric space, we usually assume that the potential facility set $T$ is given. Therefore, we can use $T$ as the candidate set, then $\eta \le 1$. Therefore, modified Algorithm 1 obtains $(1+2\rho)$-approximation for fair $k$-median problem and $2+8\rho$-approximation for fair $k$-means problem in metric space.

% Algorithm~\ref{alg-fair} can be easily extended to address $k$-median and $k$-means in general metric space.

\paragraph{$k$-Median in metric space.}Firstly, we consider fair $k$-median in general metic space. We use $\dist(\cdot, \cdot)$ to denote the distance between two points. We assume that the potential facility set $T$ is given. Therefore, in Algorithm~\ref{alg-fair}, we just use the given facility set $T$ rather than computing the approximate centroid set. The cost of fair $k$-median can be written as 
\begin{equation}
    \begin{aligned}
        \mathtt{Cost}(P, S, \phi_S^*)
        &= \sum_{p\in P}\sum_{s\in S} \mathtt{dist}( p,s) \phi_S^*(p,s). \\
        \end{aligned}
        \end{equation}
Similar to Lemma~\ref{the-ratio}, we have the following lemma.
\begin{lemma}
Let $\eta$ be any positive number. If we suppose $\mathtt{Cost}(P, T, \phi_T^*) \le \eta \cdot OPT$, then the solution $(S, \phi_S^*)$ returned by Algorithm~\ref{alg-fair} (the construction of $T$ should be slightly changed) is an $\big(\eta+(\eta + 1)\rho\big)$-approximate solution for fair $k$-median problem in metric space, where $\rho$ is the approximation ratio of vanilla $k$-median.
    \label{median-ratio}
\end{lemma}
%% \vspace{-0.1in}
\begin{proof}
% \begin{observation}
%     Every $\tilde{s} \in S_{opt}$ is the centroid of $N_O(V)$, \emph{i.e.}, 
% %$o = \frac{\sum_{p\in N_O(V)}p\cdot \phi_{S_{opt}}(p,\tilde{s})}{\sum_{p\in N_O(V)}}\phi_{S_{opt}}(p)$    
%     { $o = \frac{\sum_{p\in N_O(V)}p\cdot \phi_{S_{opt}}(p,\tilde{s})}{\sum_{i=1}^{m} w_O^{(i)}(V)}$}.
% \end{observation}
Now we consider another assignment strategy: we firstly assign $P$ to $T$ according to $\phi_T^*$ ( recall that $\phi_T^*$ is the optimal fractional assignment matrix from $P$ to $T$), and then we assign every weighted point in $T$ to some $s\in S$ such that $s$ is closest point to $\pi(t)$. Since $\phi_S^*$ is the optimal assignment matrix from $P$ to $S$, the cost of this assignment strategy should have:
\begin{equation}
    \begin{aligned}
       % &\le 
        % \sum_{i=1}^{m}\sum_{p\in P^{(i)}}\sum_{t\in T} \Vert p-\mathcal{N}(\pi(t), S)\Vert ^2 \phi_T^*(p,t) \\
        % \sum_{p\in P}\sum_{t\in T} \Vert p-\mathcal{N}(\pi(t), S)\Vert ^2\phi_T^*(p,t)&\geq \mathtt{Cost}(P, S, \phi_S^*).\\
   \mathtt{Cost}(P, S, \phi_S^*) &\le \sum_{p\in P}\sum_{t\in T} \mathtt{dist}( p,\mathcal{N}(t, S)) \phi_T^*(p,t) \\
        &\le\sum_{t\in T}\sum_{p\in P} \Big[\mathtt{dist}( p,t) + \dist( t , \mathcal{N}(t, S))\Big]\phi_T^*(p,t)\\
        &= \underbrace{\sum_{p\in P}\sum_{t\in T} \mathtt{dist} (p,t)\phi_T^*(p,t)}_{\text{(a)}} + \underbrace{\sum_{p\in P}\sum_{t\in T}\mathtt{dist}( t , \mathcal{N}(t, S))\phi_T^*(p,t)}_{(b)}.
    \end{aligned}
\end{equation}
The second inequality is triangle inequality. Then we bound (a) and (b) separately. Firstly, 
\begin{equation}
    \begin{aligned}
        (a)=\sum_{p\in P}\sum_{t\in T} \mathtt{dist}( p,t)\phi_T^*(p,t) = \mathtt{Cost}(P,T,\phi^*_{T}) \le \eta \cdot OPT
    \end{aligned}
\end{equation}
% The first inequality holds because $\pi(t)$ is the centroid of the weighted points assigned to $t$, so that $\pi(t)$ minimizes the sum of the squared distances. 
% The first inequality holds because $\pi(t)$ is the centroid of the weighted points assigned to $t$, minimizing the weighted sum of the squared distances between them.
% from points in $C_t$. 
Next, we focus on (b). 
Suppose $S_{median}$ is the optimal $k$-median solution of $T$. Then we have:
% we use $nrst_{S^*} (t)$ denote the nearest center of $t$ in $S^*$.
\begin{equation}
    \begin{aligned}
(b)&=\sum_{p\in P}\sum_{t\in T}\dist ( t, \mathcal{N}(t, S))\phi_T^*(p,t) \\
        &\le \rho \sum_{p\in P}\sum_{t\in T}\dist (t, \mathcal{N}(t, S_{median})\phi_T^*(p,t) \\
        % &= \rho \sum_{p\in P}\sum_{t\in T} \Vert \pi(t) - \mathcal{N}(\pi(t), S_{means})\Vert ^2 \phi_T^*(p,t)\\
        &= \rho \sum_{p\in P}\sum_{t\in T} \big[\sum_{\tilde{s} \in S_{opt}} \dist ( t , \mathcal{N}(t, S_{median}))\phi_{S_{opt}}^*(p,\tilde{s})\big]\phi_T^*(p,t)\\
        &\le \rho \sum_{p\in P}\sum_{t\in T} \big[\sum_{\tilde{s} \in S_{opt}} \dist (t , \tilde{s})\phi_{S_{opt}}^*(p,\tilde{s})\big]\phi_T^*(p,t).
\end{aligned}
\end{equation}
Further, according to the triangle inequality, we have

\begin{equation}
    \begin{aligned}
        (b)&\le \rho \sum_{p\in P}\sum_{t\in T} \Big[\sum_{\tilde{s} \in S_{opt}} \big[ \dist ( t , p)  + \dist( p , \tilde{s})  \big]\phi_{S_{opt}}^*(p,\tilde{s})\Big]\phi_T^*(p,t)\\
        &\le \rho \sum_{p\in P}\sum_{t\in T} \sum_{\tilde{s} \in S_{opt}}\dist( t , p)\phi_{S_{opt}}^*(p,\tilde{s})\phi_T^*(p,t) \\&+ \rho\sum_{p\in P}\sum_{t\in T} \sum_{\tilde{s} \in S_{opt}}\dist( p , \tilde{s})\phi_{S_{opt}}^*(p,\tilde{s})\phi_T^*(p,t) \\
        & = \rho \sum_{p\in P}\sum_{t\in T} \dist (t , p)\phi_T^*(p,t) + \rho\sum_{p\in P}\sum_{\tilde{s}\in S_{opt}} \dist ( p , \tilde{s})\phi_{S_{opt}}^*(p,\tilde{s}). 
        % \le (2\eta + 2)\rho \cdot OPT.
    \end{aligned}
\end{equation}
The last equality holds because for any $p\in P$, $\sum_{\tilde{s} \in S_{opt}} \phi_{S_{opt}}^*(p,\tilde{s}) = 1$ and $\sum_{\tilde{t} \in T} \phi_{T}^*(p,t) = 1$.
The first term is exactly $\rho$ times of (a) and the second term equals $\rho \cdot OPT$.
Through combining (a) and (b),  we can obtain an approximation factor of $\eta + (\eta + 1)\rho$.
\end{proof}

\paragraph{$k$-Means in metric space.} Using the same idea of Lemma~\ref{median-ratio} with squared triangle inequality $\dist^2 (a,b) \le 2\dist^2(a,c) + 2\dist^2(c,b)$, we can immediately obtain the following corollary.
\begin{corollary}
    Let $\eta$ be any positive number. If we suppose $\mathtt{Cost}(P, T, \phi_T^*) \le \eta \cdot OPT$, then the solution $(S, \phi_S^*)$ returned by Algorithm~\ref{alg-fair} (slightly changed as above) is an $\big(2\eta+(4\eta + 4)\rho\big)$-approximate solution for fair $k$-means problem in metric space, where $\rho$ is the approximation ratio of vanilla $k$-means.
\end{corollary}

When considering $k$-clustering problem in metric space, we usually assume that the potential facility set is given. We just use it as our candidate set $T$. Hence, the $\eta = 1$ in the above analysis, which leads a $(2+\rho)$-approximation for fair $k$-median and a $(2+8\rho)$-approximation for fair $k$-means.

\section{Supplementary Experiment}
\label{sup-exp}
\subsection{Datasets}
The detailed information of our datasets is shown in Table~\ref{tab-dataset}. The group partition of every dataset is based on the ``Group Column''. Every group column has some group values. The set of groups is the Cartesian product of group values of all group column. For example, the groups of \textbf{Bank} dataset are (married, yes), (married, no), (single, yes), (single, no), (divorced, yes), (divorced, no). For large dataset  \textbf{Census} and \textbf{Creditcard}, we sample 1000 points to make sure the LP solver works in acceptable time. 
% All the datasets of our experiments from UCI repository have CC BY 4.0 license.
\begin{table}[ht]
\caption{Detailed Datasets Information}
\begin{tabular}{rcccc} 
\textbf{Dataset} & \textbf{Size} & \textbf{Dimension} & \begin{tabular}{r} 
\textbf{Group Column}
\end{tabular} & \textbf{Groups Values} \\ \toprule Bank &9999 & 3 & marital & married, single, divorced \\
\cline { 4 - 5 } & & & default & yes, no \\
 \hline \multirow{2}{*}{ Adult } &4522 & 5 & sex & female, male \\
\cline { 4 - 5 } & &  & race & Amer-ind, asian-pac-isl, \\
& &  & & black, other, white \\
\hline Creditcard & 30000 & 5 & marriage & married, single, other, null \\
\cline { 4 - 5 } & &  & education & 7 groups \\
\hline Census1990 &50000 & 12 & dAge & 8 groups \\
\cline { 4 - 5 } & & & iSex & female, male \\
\hline
Moons & 200 & 2 & color & 2 groups \\
\hline
Hypercube & 200 & 3 & color & 2 groups \\
\hline
Complex & 3032 & 2 & color & 9 groups \\
\hline
Cluto & 10000 & 2 & color & 8 groups \\ \hline
Breastcancer & 570 & 31 & label & 2 groups \\ \hline
Biodeg  & 1055  & 40 & label & 2 groups \\
\bottomrule
\end{tabular}
\label{tab-dataset}
\end{table}
% \subsection{Comparison on Violations}
% We compare the violation yielded by the rounding technique in Section~\ref{sec-frac} to the vanilla $k$
% -means and NIPS19~\citep{bera2019fair} in strictly fair setting. We use the definition of violation introduced by \citep{bera2019fair}. The result shows that the violation introduced by our rounding method is $0$ in most cases and no more than $1$. And our rounding method achieve the smallest violation.
\subsection{Comparison on Cost with different \texorpdfstring{$k$}{} and \texorpdfstring{$(\alpha, \beta)$}{}}
\label{D2}

In the main paper, we set $\alpha_i = \beta_i = \frac{|P^{(i)}|}{|P|}$. Here, we try different $\alpha$ and $\beta$ to compare our algorithm to baselines. In order to make sure that the values of $\alpha$ and $\beta$ are feasible, we introduce the parameter $\delta \in (0,1)$, which represents the degree of relaxation of fairness constraints, with a larger $\delta$ indicating looser constraints. We set $\alpha_i = \frac{|P^{(i)}|}{|P|} \cdot \frac{1}{1-\delta}$ and $\beta_i = \frac{|P^{(i)}|}{|P|} \cdot {(1-\delta)}$. We set $\delta = 0.1$ and $0.2$ to compare the cost
with baselines. The results are shown in Figure~\ref{fig:delta_1} and Figure~\ref{fig:delta_2}, respectively.

{In fact, as $\delta$ increases, the fairness constraints of the $(\alpha,\beta)$-fair $k$-means problem become more relaxed, and the corresponding fair $k$-means problem approaches the vanilla $k$-means problem. In cases where $\delta$ is large, in each cluster, the legal range of points from each group is larger, making the protection of fairness constraints less important, thus resulting in the optimal fair $k$-means center positions being very close to the  centers of vanilla $k$-means. In the Table 2 of \citep{bohm2021algorithms}, it is mentioned that when $\delta=0.2$, the clustering results of vanilla $k$-means only violate the fairness constraints by 0.4\%-2\%, which makes our algorithm less advantageous under a relatively relaxed $\delta$ value.}

\begin{figure}[ht]
    \centering
    \includegraphics[width=\textwidth]{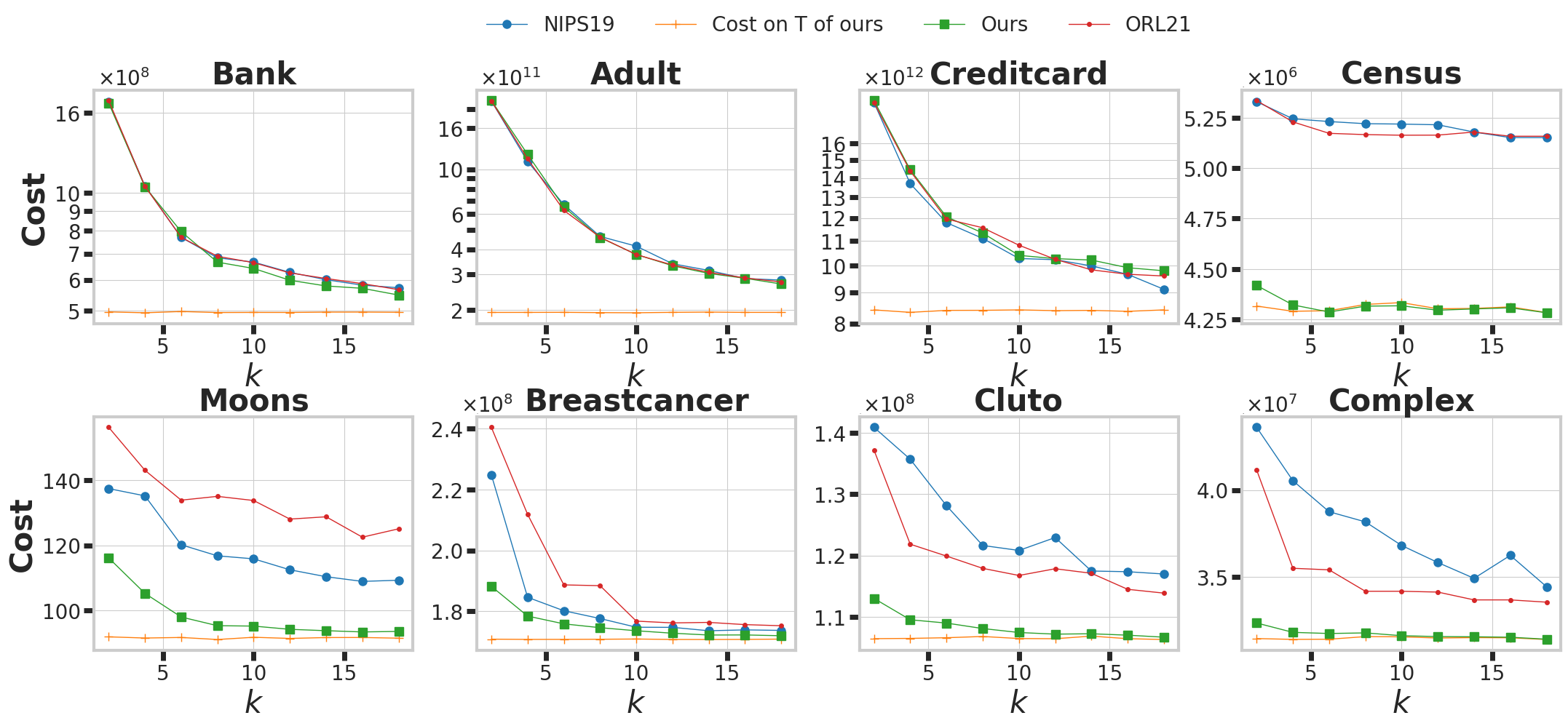}
    \caption{Comparison on Clustering Cost with $\delta = 0.1$}
    \label{fig:delta_1}
\end{figure}

\begin{figure}[ht]
    \centering
    \includegraphics[width=\textwidth]{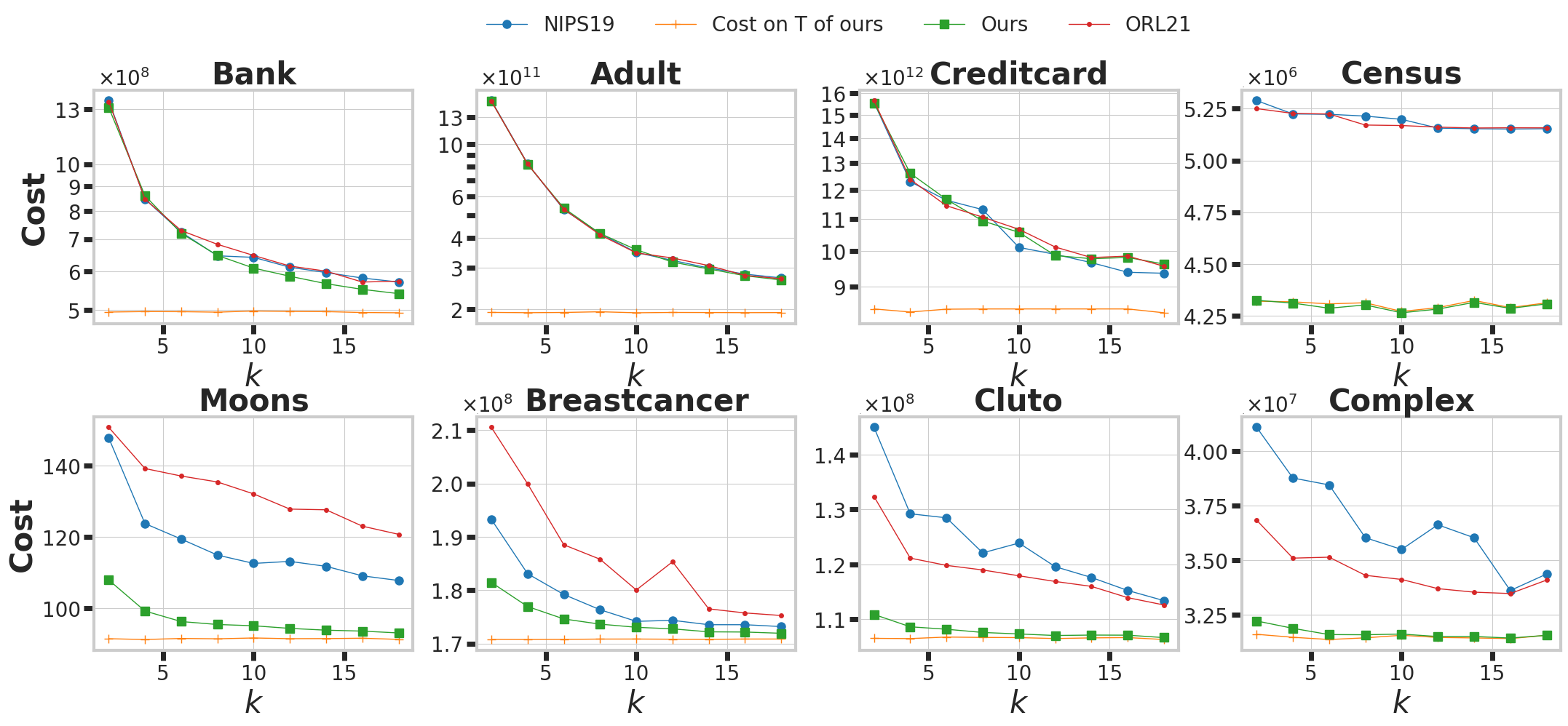}
    \caption{Comparison on Clustering Cost with $\delta = 0.2$}
    \label{fig:delta_2}
\end{figure}

\subsection{Comparison on Cost of \texorpdfstring{$k$}{}-sparse Wasserstein Barycenter}
We compare our algorithm with the very recent work~\citep{yang2024approximate} (denoted by IJCAI24) who obtain $(2+\sqrt{\rho})^2$-approximate solution of $k$-sparse WB. The results are shown in Figure~\ref{fig-wb}. In most cases, our algorithm can achieve a  10\%-30\% cost advantage over the previous work.
\begin{figure}
    \centering
    \includegraphics[width = \textwidth]{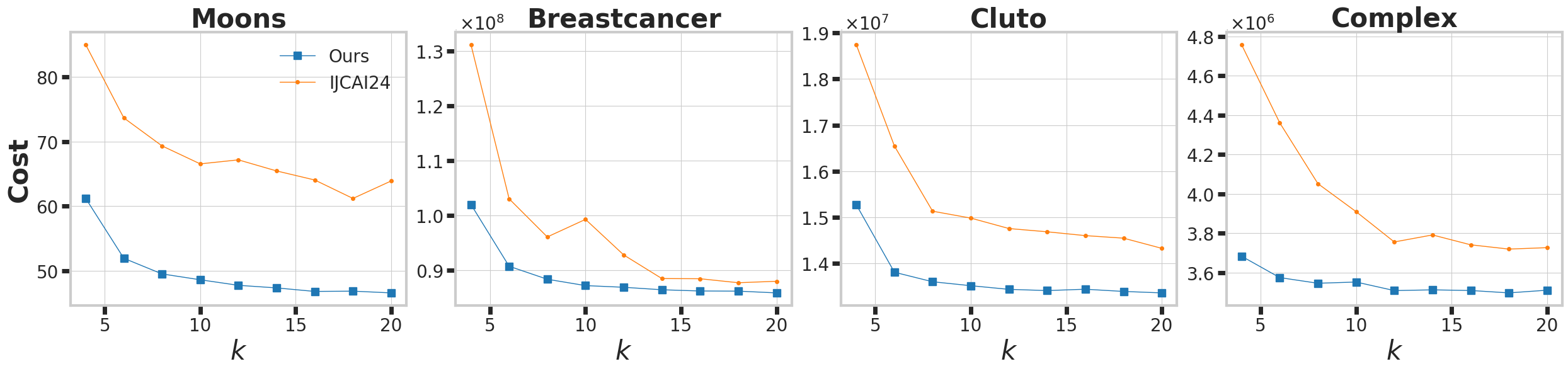}
    \caption{Comparison on the Cost of $k$-sparse Wasserstein Barycenter}
    \label{fig-wb}
\end{figure}

\subsection{Cost on Different sampling ratio}
\label{sec-sample}
In our algorithm, the most time consuming step is to solve LP(\ref{eq-fair}) on $T$. A key observation during our experiment is that, after solving LP(\ref{eq-fair}) on $T$, a large amount of points of $T$ have weight of $0$. Therefore, it is possible to reduce the size of $T$ while maintain the quality of $T$. Meanwhile, smaller $T$ helps to reduce the running time. In order to verify our thoughts, we use sampling method after we obtain $T$. We use sampling ratio of $100\%$, $50\%$, $20\%$ and $10\%$ and calculate the final cost of Algorithm~\ref{alg-fair} with different $k$. The results are shown in Figure~\ref{fig-sample-5}\ref{fig-sample-10}\ref{fig-sample-15}\ref{fig-sample-20}. In these figures, we can see that in most cases, the cost of sampled $T$ do not increase too much (50\% sample yields no more than 10\% cost increasing and even 10\% sample yields no more than 20\% cost increasing in most cases).

\subsection{Running time with different sampling ratio on \texorpdfstring{$T$}{}}
As we discussed in \ref{sec-sample}, sampling on relaxed solution $T$ can reduce the running time while the overall cost not increasing too much. We also test the running time with different sampling ratio. In summary, the running time of solving LP(\ref{eq-fair}) on $T$ and overall Algorithm~\ref{alg-fair}, shown in Table~\ref{tab-time-t} and Table~\ref{tab-time-overall}, can be significantly reduced by sampling.

\begin{table}[htbp]
\label{tab-time-t}
\caption{Time (seconds) of solving LP(\ref{eq-fair}) on $T$ with different sampling ratio}
\begin{tabular}{rcccc}
\toprule
\textbf{Dataset} & \textbf{100\%} & \textbf{50\%} & \textbf{20\%} & \textbf{10\%} \\ \midrule
Bank             & 39.97          & 19.14         & 7.12          & 3.39          \\ \hline
Adult            & 66.48          & 28.58         & 9.67          & 4.64          \\ \hline
Creditcard       & 80.235         & 32.51         & 11.08         & 5.43          \\ \hline
Census           & 76.46          & 37.78         & 13.96         & 6.64          \\ \hline
Moons            & 3.75           & 1.89          & 0.68          & 0.33          \\ \hline
Breastcancer     & 11.03          & 5.28          & 2.01          & 1.07          \\ \hline
Cluto            & 192.57         & 91.72         & 36.03         & 18.18         \\ \hline
Complex          & 49.70          & 24.74         & 9.11          & 4.54 \\ \bottomrule        
\end{tabular}
\end{table}

\begin{table}[htbp]
\label{tab-time-overall}
\caption{Overall time (seconds) with different sampling ratio of $T$ when $k=20$}
\begin{tabular}{rcccc}
\toprule
\textbf{Dataset} & \textbf{100\%} & \textbf{50\%} & \textbf{20\%} & \textbf{10\%} \\ \midrule
Bank             & 42.03          & 21.20         & 9.16          & 5.44          \\ \hline
Adult            & 69.19          & 31.24         & 12.31         & 7.34          \\ \hline
Creditcard       & 83.42         & 35.62          & 14.20         & 8.57          \\ \hline
Census           & 80.23          & 41.62         & 17.86         & 10.48          \\ \hline
Moons            & 4.07           & 2.15          & 0.97          & 0.60          \\ \hline
Breastcancer     & 11.78          & 6.05          & 2.68          & 1.67          \\ \hline
Cluto            & 201.54         & 100.64         & 45.82         & 27.74         \\ \hline
Complex          & 52.23          & 27.25         & 11.66          & 7.05 \\ \bottomrule        
\end{tabular}
\end{table}

\begin{figure}[htbp]
    \centering
    \includegraphics[width=\textwidth]{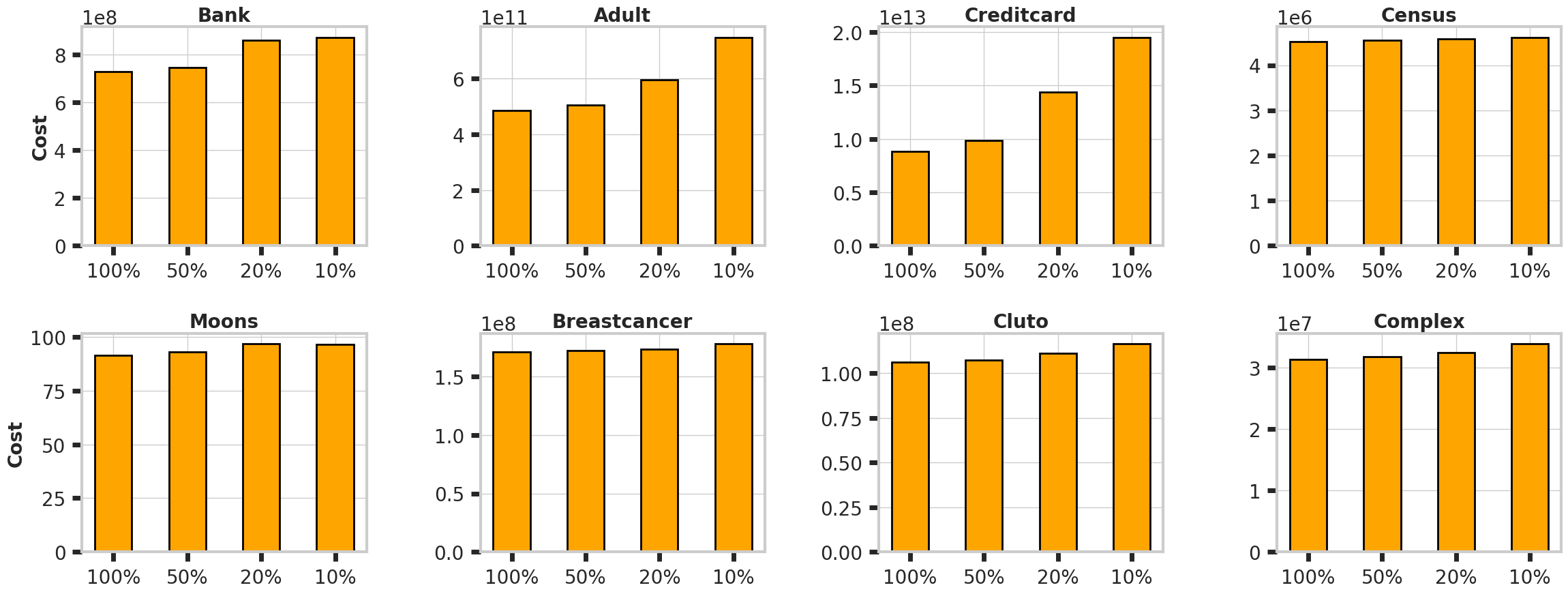}
    \caption{The cost on centriod set $T$ with different sampling ratio when $k=5$}
    \label{fig-sample-5}
\end{figure}
\begin{figure}[htbp]
    \centering
    \includegraphics[width=\textwidth]{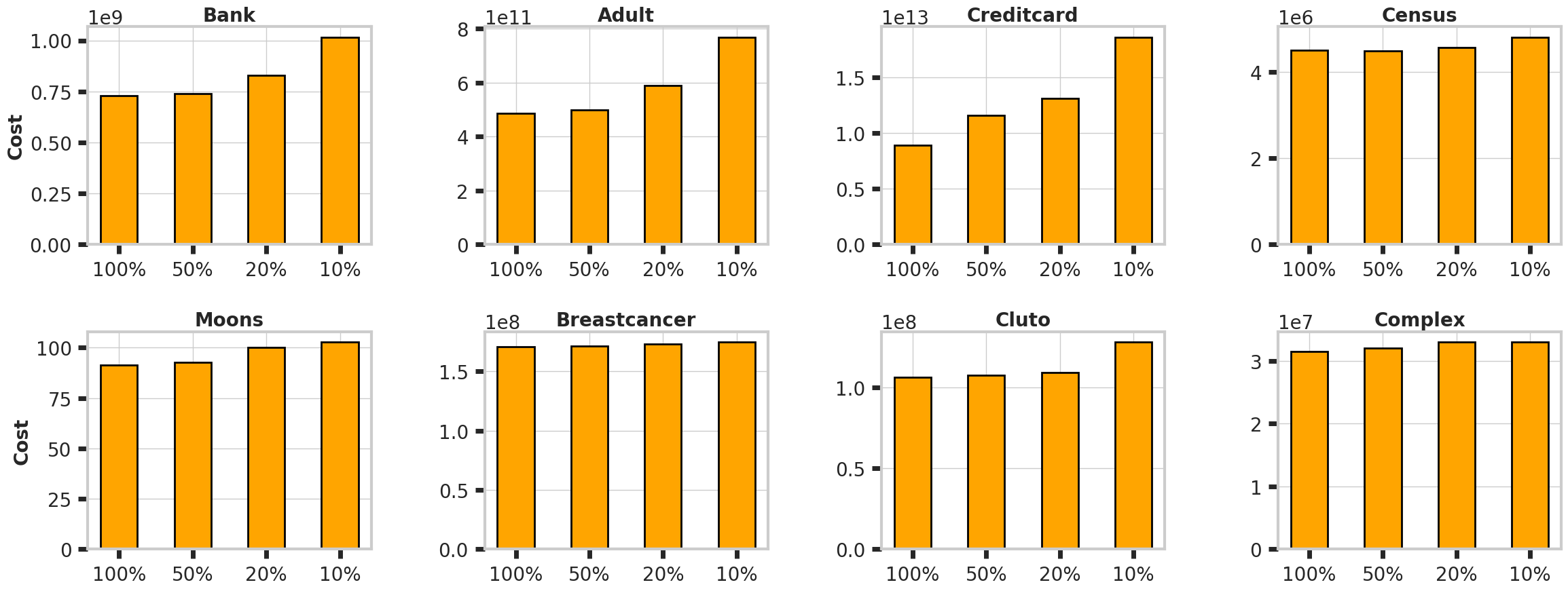}
    \caption{The cost on centriod set $T$ with different sampling ratio when $k=10$}
    \label{fig-sample-10}
\end{figure}
\begin{figure}[htbp]
    \centering
    \includegraphics[width=\textwidth]{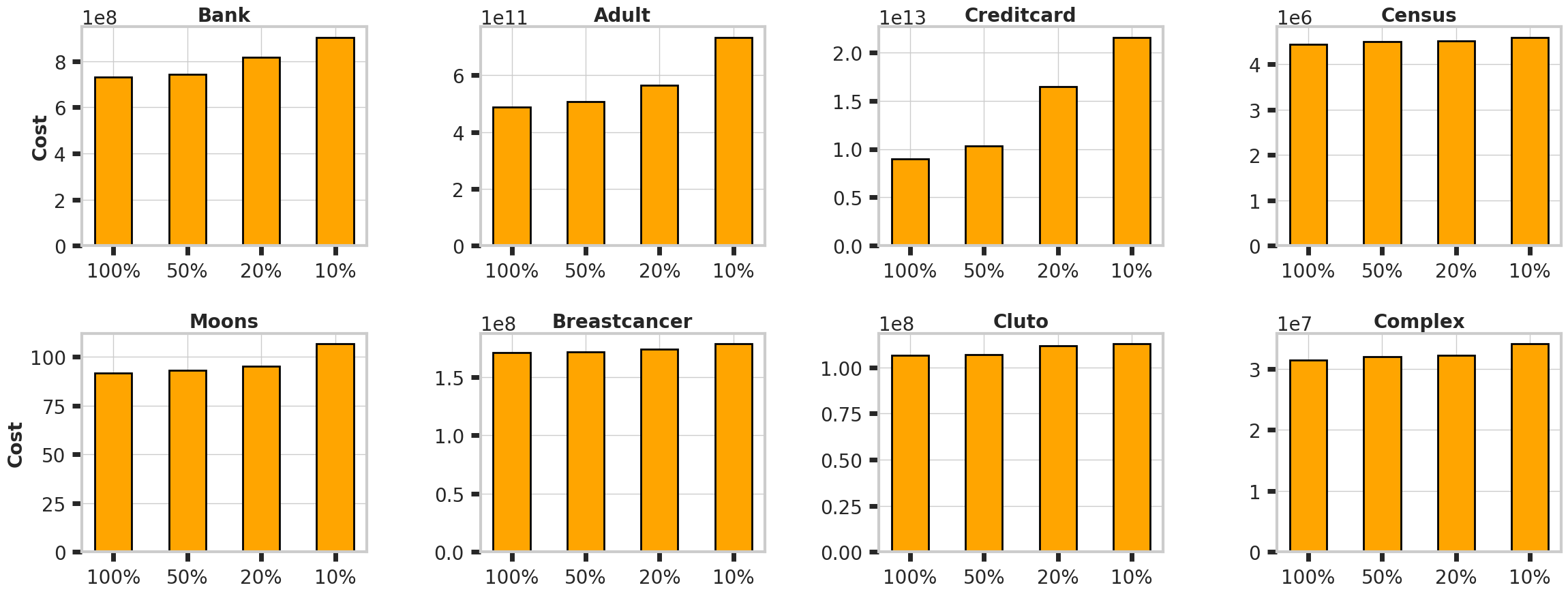}
    \caption{The cost on centriod set $T$ with different sampling ratio when $k=15$}
    \label{fig-sample-15}
\end{figure}
\begin{figure}[htbp]
    \centering
    \includegraphics[width=\textwidth]{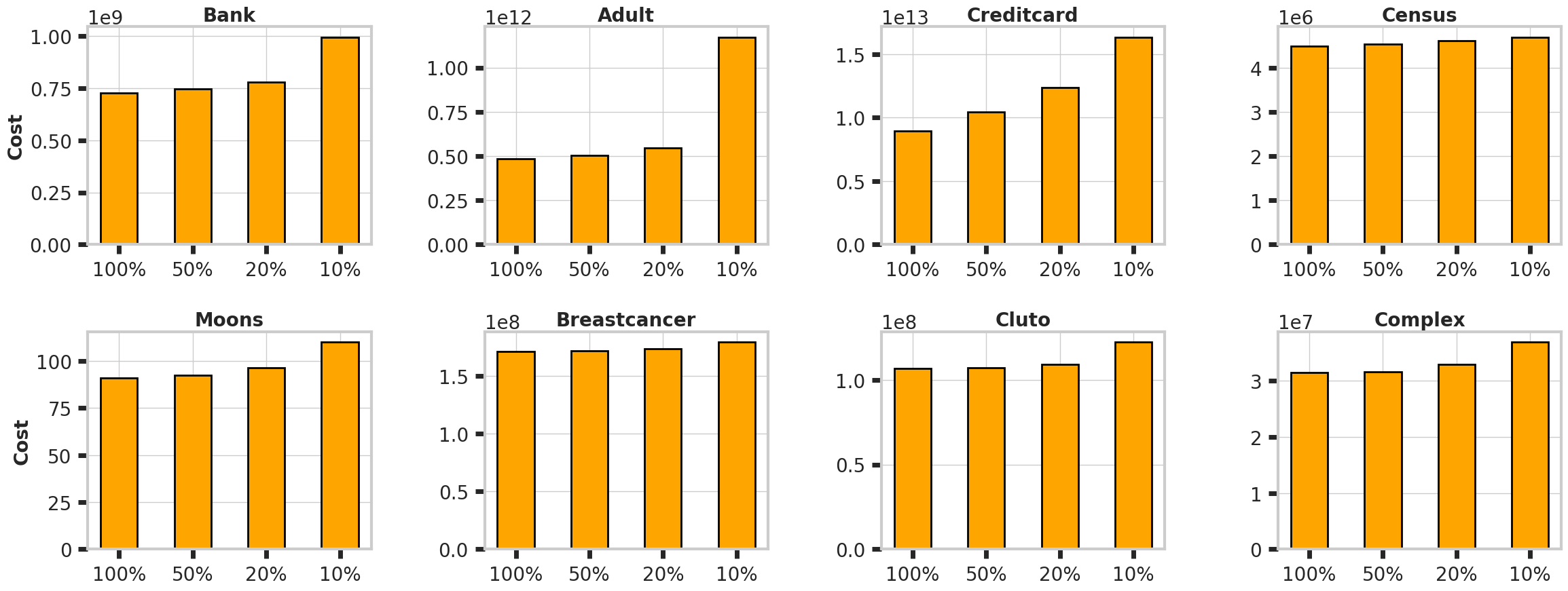}
    \caption{The cost on centriod set $T$ with different sampling ratio when $k=20$}
    \label{fig-sample-20}
\end{figure}

\subsection{Comparison of Running Time with  baselines}

We compared the running time of our algorithm {(Algorithm~\ref{alg-fair} with our rounding technique)} with the baseline  NIPS19~\citep{bera2019fair}. For strictly fair datasets, we also tested the running time of Algorithm~\ref{alg-balance} and ORL21~\citep{bohm2021algorithms}. The results are shown in Table~\ref{non-strict-time} and Table~\ref{strict-time}. {Below, we provide a detailed analysis on the comparisons.} 

\paragraph{Comparison between Algorithm~\ref{alg-fair} and NIPS19~\citep{bera2019fair}.} Algorithm~\ref{alg-fair} and NIPS19 both have two important subprocedures: linear programming and the $k$-means algorithm. These two steps are the bottlenecks for Algorithm~\ref{alg-fair} and NIPS19. {Specifically, NIPS19 first runs the $k$-means algorithm (\emph{i.e.}, $k$-means++), and then}  calls the LP solver once to compute the fractional assignment. {A different part of our Algorithm~\ref{alg-fair} is that it} calls the LP solver twice, once to compute the weights of candidate set $T$ and once to compute the fractional assignment, and calls the $k$-means algorithm once. In Algorithm~\ref{alg-fair}, we only need to run $k$-means on $T$, which should be much smaller than the whole dataset, leading to less running time for the $k$-means subprocedure compared to NIPS19. However, the first call to the LP solver to compute the weight of $T$ consumes more time than the second call because $|T| > k$ usually. We illustrate the running time of every critical subprocedure of both algorithms in Table~\ref{non-strict-time}. Our $k$-means step is faster, but we have to run an extra LP step. Therefore, the running time comparison between these two algorithms is complex. Generally speaking, LP takes more time than $k$-means, which means our Algorithm~\ref{alg-fair} usually runs slower than NIPS19. However, with the development of LP solvers, { we can expect that the runtime of Algorithm~\ref{alg-fair} could be further reduced with more advanced LP solvers.
%LP is becoming faster, allowing our Algorithm~\ref{alg-fair} to run in acceptable time.
}

\begin{table}[H]
    \centering
    \begin{tabular}{cccccccc}
        \toprule
        &       & Construct T           & LP on T               & k-means               & LP on S               & Rounding             & Total     \\ \hline
\multirow{2}{*}{Bank}         &    Algorithm\ref{alg-fair}   & 0.01   & 2.4   & $<$0.01   & 1.23   & \textless{}0.01      & 3.78      \\
               &    NIPS19    & /      & /      & 0.14   & 0.81   & \textless{}0.01      & 1.11     \\ \hline
\multirow{2}{*}{Creditcard}   &  Algorithm~\ref{alg-fair}  & 0.01   & 4.06   & $<$0.01   & 2.27   & \textless{}0.01      & 6.51      \\
               &  NIPS19  & /      & /      & 0.18   & 2.05   & \textless{}0.01      & 2.39      \\ \hline
\multirow{2}{*}{Census1990}   &  Algorithm~\ref{alg-fair} & 0.01   & 7.51   & 0.02   & 5.19   & \textless{}0.01      & 12.99      \\
               &  NIPS19  & /      & /      & 0.30   & 3.94   & \textless{}0.01      & 4.42     \\ \hline
\multirow{2}{*}{Adult}        &  Algorithm~\ref{alg-fair}  & 0.01   & 4.14   & $<$0.01   & 1.80   & \textless{}0.01      & 6.12      \\
               & NIPS19 & /      & /      & 0.18   & 1.23   & \multicolumn{1}{c}{\textless{}0.01} & 1.59     \\ \hline
\multirow{2}{*}{Breastcancer} &   Algorithm~\ref{alg-fair}     & 0.01   & 0.19   & $<$0.01   & 0.82   & \textless{}0.01      & 1.33      \\
               &    NIPS19   & /      & /      & 0.10   & 0.22   & \textless{}0.01      & 0.45     \\ \bottomrule
\end{tabular}
\caption{Running time (s) on non-strictly fair datasets}
\label{non-strict-time}
% { why lp on S are so different?}}
\end{table}

\vspace{-0.2in}

\begin{table}[H]
\caption{Running time (s) on strictly fair datasets}
\label{strict-time}
\begin{tabular}{cccccccc}
\toprule
               &       & Construct T           & LP on T               & k-means               & LP on S               & Rounding             & Total     \\ \hline
\multirow{4}{*}{Moons}        & Algorithm~\ref{alg-fair}          & 0.01   & 0.18   & $<$0.01   & 0.64   & \textless{}0.01      & 0.83      \\
               & NIPS19               & /      & /      & 0.07   & 0.70    & 0.01  & 0.78       \\
               & Algorithm~\ref{alg-balance}          & \multicolumn{1}{c}{/} & \multicolumn{1}{c}{/} & $<$0.01 & \multicolumn{1}{c}{/} & /     & \multicolumn{1}{c}{0.59} \\
               & ORL21               & \multicolumn{1}{c}{/} & \multicolumn{1}{c}{/} & 0.02 & \multicolumn{1}{c}{/} & /     & \multicolumn{1}{l}{0.48} \\ \hline
\multirow{4}{*}{Cluto}        &   Algorithm~\ref{alg-fair}  & 0.01   & 1.01   & $<$0.01   & 1.30   & \textless{}0.01      & 2.36      \\
               &    NIPS19  & /      & /      & 0.07   & 1.54   & \textless{}0.01      & 1.66      \\
               & Algorithm~\ref{alg-balance}          & \multicolumn{1}{c}{/} & / & $<0.01$ & \multicolumn{1}{c}{/} & /     & \multicolumn{1}{c}{0.56} \\
               & ORL 21               & \multicolumn{1}{c}{/} & \multicolumn{1}{c}{/} & 0.56 & \multicolumn{1}{c}{/} & /     & \multicolumn{1}{c}{0.72} \\ \hline
\multirow{4}{*}{Complex}      &   Algorithm~\ref{alg-fair}  & 0.01   & 1.08   & $<$0.01   & 0.61   & \textless{}0.01      & 1.71      \\
               &   NIPS19  & /      & /      & 0.05   & 0.72   & \textless{}0.01      & 0.79      \\
               &   Algorithm~\ref{alg-balance}   & \multicolumn{1}{c}{/} & \multicolumn{1}{c}{/} & $<0.01$ & \multicolumn{1}{c}{/} & /     & \multicolumn{1}{c}{0.58} \\
               &  ORL21  & \multicolumn{1}{c}{/} & \multicolumn{1}{c}{/} & 0.56 & \multicolumn{1}{c}{/} & /     & \multicolumn{1}{l}{0.72} \\ \hline
    \multirow{4}{*}{Hypercube}      &   Algorithm~\ref{alg-fair}  & 0.01   & 5.71   & 0.01   & 4.40   & \textless{}0.01      & 10.27      \\
               &   NIPS19  & /      & /      & 0.15   & 2.58   & \textless{}0.01      & 2.87      \\
               &   Algorithm~\ref{alg-balance}   & \multicolumn{1}{c}{/} & \multicolumn{1}{c}{/} & $<0.01$ & \multicolumn{1}{c}{/} & /     & \multicolumn{1}{c}{0.39} \\
               &  ORL21  & \multicolumn{1}{c}{/} & \multicolumn{1}{c}{/} & 0.67 & \multicolumn{1}{c}{/} & /     & \multicolumn{1}{l}{0.83} \\ \bottomrule
\end{tabular}
% \caption{Comparison of Running Time}
\end{table}

\paragraph{Discussion on the construction of $T$.} According to Algorithm~\ref{alg-fair}, $T$ should be an approximate centroid set~\citep{matouvsek2000approximate}. {Thanks to the open-source project by \citep{kanungo2002local},} which provides an efficient implementation of the approximate centroid set, we used their algorithm as part of our procedure in our code. \citet{kanungo2002local} used a sampling technique, leading to a trade-off between performance and efficiency. In our experiment, we sampled 10\% of points in the approximate centroid set as $T$. A higher sample rate yields better performance (lower cost) but longer running time. 

{Besides, an implicit benefit of the construction of $T$ is that it is irrelevant to the parameters $k$, $\alpha$, and $\beta$. So if we consider a real scenario that we need to repeatedly try different choices for these parameters (e.g., we may want to tune the value $k$ and select the most satisfying result), 
the step of constructing $T$ and performing linear programming on $T$ can be seen as preprocessing of datasets before the tuning. Namely, we just need to run this preprocessing one time, and consequently the amortized cost over the whole tuning procedure 
%because it is irrelevant to the parameters $k$, $\alpha$, and $\beta$. Therefore, if we compute $T$ in advance, our algorithm is conducive to parameter tuning, 
%consequently the average running time 
can be reduced significantly.
}

\paragraph{Running time comparison on strictly fair datasets.} For strictly fair datasets, 
we consider Algorithm~\ref{alg-fair}, NIPS19, Algorithm~\ref{alg-balance}, and {ORL21~\cite{}}. Algorithm~\ref{alg-balance} has an advantage in efficiency in most datasets. The primary reason is that Algorithm~\ref{alg-balance} only calls the $k$-means algorithm once and does not need to solve the LP. As for ORL21, it needs to run $k$-means for each group and then choose the best one. As a result, ORL21 takes longer time than Algorithm~\ref{alg-balance}, especially on the datasets with large number of groups.

\subsection{Experiments of Our Rounding Algorithm}

In this section, we implement our rounding algorithm in Appendix~\ref{rounding-alg} and compute the violation factor across different datasets and parameters. For convenience, we parameterize $\alpha_i$ and $\beta_i$ for the $i$-th group using a single parameter $\delta$. Specifically, we set $\beta_i = \frac{| P^{(i)} | (1-\delta)}{| P | }$ and $\alpha_i = \frac{| P^{(i)} |}{| P | (1-\delta)}$. Generally speaking, the smaller the $\delta$, the stricter the fairness constrains are. In Table~\ref{violation0}~\ref{violation01}~\ref{violation02}, the violation introduced by our rounding algorithm is less than 1 in most of the cases and never exceeds 2, which aligns with our theoretical analysis.

\begin{table}[ht]
\caption{Violation factor of our rounding algorithm with different $k$ ($\delta = 0$)}
\begin{tabular}{ccccccccccccc}
\toprule
\textbf{dataset}      & k=2    & 4    & 6    & 8    & 10   & 12   & 14   & 16   & 18   & 20   & 25   & 30   \\ \hline
Moons        & 0    & 0    & 0    & 0    & 0    & 0    & 0    & 0    & 0    & 0    & 0    & 0    \\ \hline
Hypercube    & 0    & 0    & 0    & 0    & 0    & 0    & 0    & 0    & 0    & 0    & 0    & 0    \\ \hline
Complex      & 0.82 & 0.89 & 0.5 & 0.83 & 0.96 & 0.95 & 0.87    & 0.95 & 0.91 & 0.85 & 0.80 & 0.89 \\ \hline
Cluto        & 0.80 & 0.86 & 0.72 & 1.01 & 1.04 & 0.94  & 1.0 & 1.02 & 0.90  & 0.90 & 1.1 & 0.9 \\ \hline
Biodeg       & 0.05 & 0.66 & 0.65 & 0.63 & 0.64 & 0.62 & 0.63 & 0.68 & 0.77 & 0.79 & 0    & 0.01 \\ \hline
Breastcancer & 0.33 & 0.34 & 0.13 & 0.69 & 0.87 & 0.90 & 0.35 & 0.94 & 0.78 & 0.76 & 0.76 & 0.18 \\ \bottomrule
\end{tabular}
\label{violation0}
\end{table}

\begin{table}[ht]
\caption{Violation factor of our rounding algorithm with different $k$ ($\delta = 0.1$)}
\begin{tabular}{ccccccccccccc}
\toprule
\textbf{dataset} & k=2    & 4    & 6    & 8    & 10   & 12   & 14   & 16   & 18   & 20   & 25   & 30   \\ \hline
Moons            & 0    & 0.3  & 0.35 & 0.40 & 0.30 & 0.40 & 0.70 & 0.5  & 0.35 & 0    & 0.20 & 0.40 \\ \hline
Hypercube        & 0    & 0.94 & 0.98 & 0.94 & 0.83 & 0.95 & 0.85 & 0.91 & 0.80 & 0.88 & 1.02 & 0.83 \\ \hline
Complex          & 0.67 & 0.98 & 0.66 & 0.87 & 0.88 & 0.97 & 0.76 & 0.77 & 0.89 & 0.97 & 0.67 & 1.03 \\ \hline
Cluto            & 0.38 & 1.05 & 0.99 & 0.83 & 0.96 & 0.94 & 0.95 & 0.93 & 0.94 & 0.91 & 0.57 & 0.99 \\ \hline
Biodeg           & 0    & 0.01 & 0.33 & 0.79 & 0.38 & 0.37 & 0.59 & 0.38 & 0.78 & 0.51 & 0.78 & 0.80 \\ \hline
Breastcancer     & 0.18 & 0.23 & 0.40 & 0.23 & 0.39 & 0.89 & 0.53 & 0.33 & 0.47 & 0.51 & 0.34 & 0.68 \\ \bottomrule
\end{tabular}
\label{violation01}
\end{table}

\begin{table}[ht]
\caption{Violation factor of our rounding algorithm with different $k$ ($\delta = 0.2$)}
\begin{tabular}{ccccccccccccc}
\toprule
\textbf{dataset} & k=2    & 4    & 6    & 8     & 10   & 12    & 14   & 16   & 18   & 20   & 25   & 30   \\ \hline
Moons            & 0    & 0.20 & 0.40 & 0.40  & 0.40 & 0.40  & 0.60 & 0.60 & 0.40 & 0.40 & 0.60 & 0.80 \\ \hline
Hypercube        & 0    & 0.56 & 0.69 & 0.88  & 0.90 & 1.125 & 0.80 & 0.91 & 0.90 & 0.97 & 0.90 & 0.90 \\ \hline
Complex          & 0.92 & 1.02 & 0.92 & 0.768 & 0.96 & 1.01  & 0.95 & 0.79 & 0.88 & 0.90 & 1.04 & 1.01 \\ \hline
Cluto            & 0.85 & 0.90 & 0.90 & 0.88  & 0.83 & 1.024 & 0.85 & 0.86 & 0.90 & 0.88 & 0.96 & 1.00 \\ \hline
Biodeg           & 0    & 0.50 & 0.56 & 0.39  & 0.51 & 0.69  & 0.19 & 0.57 & 0.56 & 0.64 & 0.75 & 0.65 \\ \hline
Breastcancer     & 0    & 0.26 & 0.42 & 0.26  & 0.69 & 0.39  & 0.29 & 0.67 & 0.80 & 0.81 & 0.85 & 0.68 \\ \bottomrule
\end{tabular}
\label{violation02}
\end{table}

% \subsection{Name of the licenses for existing assets}
% \label{sec-lic}
% All the datasets of our experiments from UCI repository, including \textbf{Breastcancer}, \textbf{Biodeg}, \textbf{Bank}. \textbf{Census1990}, \textbf{Adult} have CC BY 4.0 license. \textbf{Moons} dataset has BSD license. We use the educational version of the optimizer Gurobi~\citep{gurobi} with an educational license (non commercial).

% \subsection{Experiment Statistical Significance}
% \label{sec-stat}

% On the dataset that appears multiple times in our experiments, we run an experiment of statistical significance as shown in Figure~\ref{fig:stat2}\ref{fig:stat}. We use the same setting as the first experiment in Section~\ref{sec-exp} with 2-sigma error bars (10 executions each marked point). The experimental results indicate that our algorithm significantly outperforms the baselines.

% \vspace{-0.1in}
% \begin{figure}[h]
%     \centering
%     \includegraphics[width=\textwidth]{error_2.png}
    
%     \caption{{The cost obtained by the algorithms with different $k$ (with 2-sigma error bars).}}
%     \label{fig:stat2}
% \end{figure}

% \begin{figure}[h]
%     \centering
%     \includegraphics[width=\textwidth]{stat.png}
    
%     \caption{{The cost obtained by the algorithms with different $k$ (with 2-sigma error bars).}}
%     \label{fig:stat}
% \end{figure}

\end{document}